\documentclass{article}
\usepackage{authblk}
\usepackage{natbib}
\usepackage[margin=1in]{geometry}
\usepackage[utf8]{inputenc} 
\usepackage[T1]{fontenc}    
\usepackage{hyperref} 
\hypersetup{
     colorlinks   = true,
     linkcolor = black,
     urlcolor    = teal,
	 citecolor = teal,
  hypertexnames=false
}
\usepackage{url}            
\usepackage{booktabs}       
\usepackage{amsfonts}       
\usepackage{nicefrac}       
\usepackage{microtype}      
\usepackage{xcolor}         

\usepackage{amsmath,amsfonts,amssymb}
\usepackage{amsthm}
\usepackage{bm}
\usepackage{enumitem,array,booktabs}
\usepackage{algorithm}
\usepackage{algorithmic}
\newcounter{protocol}
\makeatletter
\newenvironment{protocol}[1][htb]{%
  \let\c@algorithm\c@protocol
  \renewcommand{\ALG@name}{Protocol}
  \begin{algorithm}[#1]%
  }{\end{algorithm}
}
\makeatother

\usepackage{natbib}
\usepackage{bbm}
\usepackage{tabulary,multirow}
\usepackage{makecell}
\usepackage{dsfont}
\usepackage{tcolorbox}
\usepackage{thmtools}
\usepackage{xpatch}
\xpatchcmd{\proof}{\itshape}{\scshape}{}{}
\renewenvironment{proof}[1][]{\par\noindent{\bfseries Proof #1\ }}{\hfill\qedsymbol\\[2mm]}
\usepackage{tablefootnote}

\newtheorem{theorem}{Theorem}
\newtheorem{definition}[theorem]{Definition}
\newtheorem{lemma}[theorem]{Lemma}
\newtheorem{proposition}[theorem]{Proposition}

\newtheorem{remark}[theorem]{Remark}

\newcommand{\restatableeq}[3]{\label{#3}#2\gdef#1{#2\tag{\ref{#3}}}}

\usepackage{xspace}
\definecolor{darkgray}{RGB}{105,105,105}
\definecolor{lightgray}{RGB}{169,169,169}
\definecolor{darkmeganta}{rgb}{0.6,0,0.6}

\DeclareMathOperator*{\argmin}{arg\,min}

\DeclareMathOperator{\loglog}{loglog}

\newcommand{\Unif}{\mathrm{Unif}}

\newtheorem{claim}{Claim}

\newcommand{\NN}{{\mathbb N}}
\newcommand{\1}[1]{\mathds{1}(#1)}
\newcommand{\ind}[1]{\mathds{1}_{#1}}

\newcommand{\EEs}[2]{\mathbb{E}_{#1}\left[#2\right]}

\newcommand{\PPs}[2]{\mathbb{P}_{#1}\left(#2\right)}
\renewcommand{\Pr}{\mathbb{P}}

\newcommand{\abs}[1]{\left|#1\right|}

\newcommand{\cA}{\mathcal{A}}

\newcommand{\cD}{\mathcal{D}}
\newcommand{\cE}{\mathcal{E}}

\newcommand{\cG}{\mathcal{G}}
\newcommand{\cH}{\mathcal{H}}

\newcommand{\cL}{\mathcal{L}}

\newcommand{\cX}{\mathcal{X}}
\newcommand{\X}{\cX}
\newcommand{\cY}{\mathcal{Y}}

\renewcommand{\epsilon}{\varepsilon}
\renewcommand{\hat}{\widehat}
\renewcommand{\tilde}{\widetilde}
\renewcommand{\bar}{\overline}

\newcommand{\nothere}[1]{}

\newcommand{\vcd}{\mathrm{VCdim}}
\newcommand{\ld}{\mathrm{Ldim}}

\newcommand{\lstr}{\ell^{\textrm{str}}}
\newcommand{\Lstr}{\cL^{\textrm{str}}}
\newcommand{\hatlstr}{\hat \cL^{\textrm{str}}}
\newcommand{\lgraph}{\cL_{\textrm{neighborhood}}}
\newcommand{\lproxy}{\cL_{\textrm{proxy}}}
\newcommand{\hatlproxy}{\hat\cL_{\textrm{proxy}}}
\newcommand{\hatlgraph}{\hat\cL_{\textrm{neighborhood}}}

\newcommand{\regret}{\textrm{Regret}}

\date{}
\title{Learnability Gaps of Strategic Classification}

\author{
  Lee Cohen\thanks{Stanford. Email: \href{mailto:leecohencs@gmail.com}{\nolinkurl{leecohencs@gmail.com}}. Authors are ordered alphabetically.}\hspace*{1cm}
  Yishay Mansour\thanks{Tel Aviv University and Google Research. Email: \href{mailto:mansour.yishay@gmail.com}{\nolinkurl{mansour.yishay@gmail.com}}.}\hspace*{1cm}
  Shay Moran\thanks{Departments of Mathematics, Computer Science, and Data and Decision Sciences, Technion and Google Research. 
  Email: \href{mailto:smoran@technion.ac.il}{\nolinkurl{smoran@technion.ac.il}}.}\hspace*{1cm}
  Han Shao\thanks{Toyota Technological Institute of Chicago. Email: \href{mailto:han@ttic.edu}{\nolinkurl{han@ttic.edu}}.}}

\begin{document}

\maketitle

\begin{abstract}%

In contrast with standard classification tasks, strategic classification involves agents strategically modifying their features in an effort to receive favorable predictions. For instance, given a classifier determining loan approval based on credit scores, applicants may open or close their credit cards and bank accounts to fool the classifier. The learning goal is to find a classifier robust against strategic manipulations. Various settings, based on what and when information is known, have been explored in strategic classification. In this work, we focus on addressing a fundamental question: the learnability gaps between strategic classification and standard learning. 

We essentially show that any learnable class is also strategically learnable: we first consider a fully informative setting, where the manipulation structure (which is modeled by a manipulation graph $G^\star$) is known and during training time the learner has access to both the pre-manipulation data and post-manipulation data. We provide nearly tight sample complexity and regret bounds, offering significant improvements over prior results. Then, we relax the fully informative setting by introducing two natural types of uncertainty. 

First, following~\cite{ahmadi2023fundamental}, we consider the setting in which the learner only has access to the post-manipulation data. We improve the results of~\cite{ahmadi2023fundamental} and close the gap between mistake upper bound and lower bound raised by them. 

Our second relaxation of the fully informative setting introduces uncertainty to the manipulation structure. That is, we assume that the manipulation graph is unknown but belongs to a known class of graphs. We provide nearly tight bounds on the learning complexity in various unknown manipulation graph settings. Notably, our algorithm in this setting is of independent interest and can be applied to other problems such as multi-label learning.

\end{abstract}

\section{Introduction}


Strategic classification lies at the intersection of machine learning and game theory, addressing the intricate challenges arising from the strategic behavior of self-interested agents seeking to manipulate classifiers for personal advantage \citep{bruckner2011stackelberg, hardt2016strategic}. This emerging field has attracted significant attention due to its relevance in real-world applications, such as loan approvals and college admissions.
For example, in loan applications, individuals might strategically try to increase their credit scores by acquiring additional credit cards, or by changing jobs to increase their salaries. Within the context of college admissions, students may opt for less challenging courses to improve their GPA, retake the SAT, or even transfer schools, all to improve their chances of acceptance. 
This introduces a more complex environment, where agents strategically adapt their features to influence the outcome of the current classifier while preserving their true underlying qualifications (e.g., being successful in college or repaying a loan). The foundational challenge in strategic classification resides in selecting a classifier that is robust to manipulations, ensuring the learning process remains resilient to strategic behavior.

When modeling strategic classification, the manipulation power of the agents is considered limited. Our approach to representing feasible manipulations involves a manipulation graph—a concept initially introduced in \cite{zhang2021incentive} and subsequently studied in, e.g., \cite{lechner2022learning, lechner2023strategic, ahmadi2023fundamental}. 
In this model, each agent has a feature vector $x$, represented as a node in a graph, and a binary label, $y$. An arc from $x$ to~$x'$ indicates that an agent with the feature vector $x$ can adjust their feature vector to $x'$—essentially, an arc denotes a potential manipulation.
Similar to the assumptions made in \cite{ahmadi2023fundamental}, we posit that the manipulation graph has a bounded degree, $k$. This constraint adds a realistic dimension to the model, reflecting the limited manipulation capabilities of the agents within the strategic classification framework. In addition, as we discuss in Theorems~\ref{thm:fi-pac-lb} and~\ref{thm:fi-online-lb}, this bounded degree constraint is also necessary in the sense that relaxing it leads to impossibility results for extremely easy-to-learn hypothesis classes. 

Consider a sequence of agents arriving and interacting with the learner. 
At each time step, the learner selects a hypothesis, and the current agent manipulates its feature vector $x$ in accordance with this hypothesis. Each agent is aiming to secure a positive classification, and thus an agent $x$ will manipulate to node $x'$ only when such manipulation alters their predicted value from negative to positive.
Importantly, $y$, the true label of the agent, does not depend on the current predictor.
Also note that the learner does not have the ability to counterfactually inquire about the manipulations the previous agents would have adopted under a different hypothesis. 

In this work, we aim to address a fundamental question:
\begin{center}
    \textbf{Does learnability imply strategic learnability?}
\end{center}
We study the above question within both PAC and online learning frameworks, addressing it across both realizable and agnostic contexts.  In each case, we analyze various levels of information received by the learner as feedback. 

%
\vspace{1mm}

We start with a \textit{fully informative} setting.  Here, the learner has full information of both the manipulation graph as well as the pre- and post-manipulation features: at each time step the learner (1) observes the pre-manipulation features of the current agent, (2)  selects and implements a hypothesis 
(according to which the agent manipulates), and (3) observes both the manipulated features and the true label of the agent. 

We are mainly interested in this model because it is the simplest and therefore a natural first step from a theoretical perspective. Nevertheless, it still could be interesting from a practical perspective because it aligns with real-world situations like customers inquiring about mortgage rates and providing their credit score or income as input, or in college admissions where colleges can access sophomore-year grades. 
This basic setting already deviates from the standard learning paradigm, as the learner needs to consider the potential manipulation of the current agents. 


\vspace{1mm}

We consider the fully informative setting as a baseline. The other two settings we consider deviate from this baseline by introducing two pragmatic elements of uncertainty regarding the learner's knowledge: 
one concerning the pre-manipulation features (previously studied for finite hypotheses class in \cite{ahmadi2023fundamental}), and the other regarding the manipulation graph (this type of uncertainty is introduced in this work).

\vspace{1mm}

In the second setting, \textit{post-manipulation feedback}, the learner selects the hypothesis before observing any information about the agent. Then, after the agent's manipulation and its classification, 
the learner observes either (i) the original feature vector, or (ii) the manipulated feature vector (we consider both variants). 
Thus, significantly less information is provided to the learner compared to the first setting, wherein the learner could compute the set of potential manipulations for each feature vector $x$ for every hypothesis. 
This setting was previously studied in \cite{ahmadi2023fundamental} who focused on finite hypotheses classes. 


We highlight that while we assume knowledge of the manipulation graph for the first two settings, our algorithms exclusively depend on local information concerning the manipulation set of the current feature vector. Thus, our algorithms remain applicable in scenarios where global information may not be readily available.

\vspace{1mm}

In the last setting, the \textit{unknown manipulation graph}, we relax the fully informative baseline by introducing uncertainty about the manipulation graph. So, the learner observes both the pre- and post-manipulation feature vectors, but no longer has access to the entire manipulation graph. Instead, the learner has prior knowledge in the form of a finite class of graphs $\cG$. The class $\cG$ is assumed to contain the true manipulation graph in the realizable setting, and in the agnostic setting, $\cG$ serves as a comparator class for competitive/regret analysis.
This setting is a natural variant of the 
fully informative feedback model. For instance, a lender can observe the credit score of a customer but does not know the (at most $k$) easiest and most effective types of manipulations for the agent to pursue —whether they involve obtaining more credit cards, increasing their salary, or even changing employment status from unemployed to employed. 


From a technical perspective, the last setting gives rise to a natural graph learning task: consider an unknown target graph. At training time, we observe random examples of the form (vertex $v$, and a random arc from $v$). During testing, we receive a random vertex $u$ (drawn from the same distribution as the training nodes), and our objective is to predict the entire neighborhood of $u$. 
This problem is motivated by applications such as recommendation systems like Netflix, where a learner observes a random user (vertex) and a random movie they watched (random arc). During testing, the learner observes a user and aims to predict a small set of movies to recommend to this user.
Our algorithm in the unknown manipulation graph setting can be applied to solve this neighborhood learning problem.


\subsection{Results Overview} 

In classical binary classification, it is well-known that Probably Approximately Correct (PAC) learnability is characterized by the Vapnik–Chervonenkis (VC) dimension \citep{vapnik:74} and that online learnability is characterized by the Littlestone dimension \citep{littlestone1988learning} (see Appendix~\ref{app:background} for their definitions). As detailed below, we essentially show that \emph{every learnable class is also strategically learnable}. For certain scenarios, we efficiently transform standard learning algorithms into their strategic counterparts, while for others, we develop entirely new algorithms. Please see Table~\ref{tb:results} for a summary according to different settings. 

\vspace{-0.5em}
\paragraph{Fully Informative Setting}
We establish upper and lower bounds for the sample complexity and mistake bounds, which are tight up to logarithmic factors. 
In particular, our bounds imply that in this basic setting, learnability implies strategic learnability. Quantitatively, we prove that the learning complexity of strategic learning is larger by at most a multiplicative logarithmic factor of the maximal degree of the manipulation graph. We further prove that this logarithmic factor is necessary.


In this setting, both PAC and online learnability are captured by the VC and Littlestone dimension of the corresponding classes of strategic losses (i.e.\ the 0-1 loss under manipulations). In the PAC setting we show that the strategic VC dimension is $\tilde \Theta(d_1\log k)$, where $d_1$ is the standard VC dimension of~$\cH$. This bound is tight up to a logarithmic factor in $d_1$ (see Theorems~\ref{thm:fi-pac} and~\ref{thm:fi-pac-lb}). 
In the online setting we provide a tight bound of $d_2\log k$ on the strategic Littlestone dimension of $\cH$, where $d_2$ is the standard Littlestone dimension of $\cH$ (see Theorems~\ref{thm:fi-online} and ~\ref{thm:fi-online-lb}).

Algorithmically, in the PAC setting any empirical risk minimizer w.r.t.\ the strategic losses achieves near optimal sample complexity in both the realizable and agnostic setting. In the online setting we construct an efficient transformation that converts an online algorithm in the standard setting to a strategic online algorithm. Our transformation involves using the algorithm from the standard setting to generate multiple experts, and running a weighted-majority vote scheme over these experts.

\vspace{-0.5em}
\paragraph{Post-Manipulation Feedback Setting}

In the PAC setting, we derive sample complexity bounds identical to those that appear in the fully informative setting. 
This occurs because, when the graph is known, the timing of observing the original feature---whether beforehand or afterward---does not make any difference. Even if the learner can only observe manipulated features, by implementing the all-positive or all-negative hypothesis at each round, one can obtain the same information feedback as observing original features.

However, in online learning,  this setting is more subtle because we cannot separate exploration (training) from exploitation (test). To handle this we again devise an efficient transformation that that converts online learners from the classical setting to strategic online learners.
Our algorithm here is similar to the one from the fully informative setting, yet the optimal mistake bound here is $d_2 k$, which has an exponentially worse dependence on $k$ (see Theorems~\ref{thm:fu-online} and~\ref{thm:fu-online-lb}). This optimal mistake bound also resolves an open question posed by \cite{ahmadi2023fundamental}.
 

\vspace{-0.5em}
\paragraph{Unknown Manipulation Graph Setting}
We answer our overarching question affirmatively for this setting too. 
The main challenge of PAC strategic learning with an unknown graph setting is that we cannot obtain an unbiased estimate of the loss of each graph, and simply running ERM does not work. We tackle this challenge by using the expected degree as a regularizer over the graphs.
Then, in the realizable setting, by finding the consistent minimal empirical degree graph-hypothesis pair, we can obtain sample complexity 
$\tilde O({(d_1\log k + k\log |\cG|)}/{\epsilon})$ (Theorem~\ref{thm:ug-pac-rel}).

In the agnostic PAC setting, although we cannot obtain an unbiased estimate of the graph loss, we propose a proxy loss, which can approximate the graph loss up to a multiplicative factor $k$. By minimizing the empirical proxy loss to find an approximate graph and running ERM over the strategic loss under this approximate graph, we are able to obtain a hypothesis that is at most a factor of $k$ from the best hypothesis-graph pair (see Theorem~\ref{thm:ug-pac-agn}).

In the realizable online setting, we design an algorithm that runs the online learning algorithm from the post-manipulation feedback setting over the majority vote of the graphs and achieves a mistake bound of $O(d_2k\log(k)+ \log |\cG|)$ (see Theorem~\ref{thm:ug-online}). In both realizable PAC and online settings, we provide a lower bound with logarithmic dependency on $|\cG|$, demonstrating that $\log|\cG|$ is inherent (see Theorems~\ref{thm:ug-pac-rel-lb} and~\ref{thm:ug-online-lb}).
We further discuss the application of our results in the unknown graph setting to the multi-label learning problem, which is of independent interest. 

\begin{table}[t]
\centering
\resizebox{\textwidth}{!}{%
\renewcommand{\arraystretch}{2}
\begin{tabular}{|c|c|c|c|}
\hline
\multirow{2}{*}{Setting}& \multicolumn{2}{c|}{Sample complexity} & Regret\\\cline{2-4}
 & Realizable & Agnostic  & Realizable$^{(*)}$   \\ \hline
\makecell{Fully informative} & \multirow{2}{*}{\makecell{
$\tilde O(\frac{d_1\log k}{\epsilon})$ 
(Thm~\ref{thm:fi-pac})\\
$\Omega(\frac{d_1\log k}{\epsilon})$ (Thm~\ref{thm:fi-pac-lb})
}}  & \multirow{2}{*}{\makecell{
$\tilde O(\frac{d_1\log k}{\epsilon^2})$ (Thm~\ref{thm:fi-pac})\\
$\Omega(\frac{d_1\log k}{\epsilon^2})$ (Thm~\ref{thm:fi-pac-lb})}} & \makecell{$O(d_2 \log k)$
(Thm~\ref{thm:fi-online})\\ $\Omega(d_2\log k )$ (Thm~\ref{thm:fi-online-lb})} \\ \cline{1-1}\cline{4-4}
\makecell{Post-manipulation\\ feedback} &   &   & \makecell{ $\tilde O(d_2 k)$ (Thm~\ref{thm:fu-online})\\
$\Omega(d_2k)$  
 (Det, Thm~\ref{thm:fu-online-lb})}  \\ \hline
\makecell{Unknown\\ manipulation graph} & \makecell{$\tilde O(\frac{d_1\log k + k\log(|\cG|)}{\epsilon})$ (Thm~\ref{thm:ug-pac-rel})\\ $\Omega(\log|\cG|)$ (Ada, Thm~\ref{thm:ug-pac-rel-lb})} & \makecell{$O(\frac{k^2\log(|\cG|) + d_1\log k}{\epsilon^2})$\\
 (Mul, Thm~\ref{thm:ug-pac-agn})} & \makecell{$O(d_2k\log k+ \log |\cG|)$ (Thm~\ref{thm:ug-online})\\
$\tilde\Omega(\log|\cG|)$ (Det, Thm~\ref{thm:ug-online-lb})} \\ \hline
\end{tabular}
}
\caption{Summary of results, where $k$ is the upper bound of the maximum out-degree of the underlying manipulation graph $G^\star$ as well as all graphs in $\cG$, $d_1 = \vcd(\cH)$ is the VC dimension of $\cH$, and $d_2 = \ld(\cH)$ is the Littlestone dimension of $\cH$. \textit{Det} means a lower bound for all deterministic algorithms. \textit{Ada} means in this lower bound, the true graph is generated adaptively. \textit{Mul} means that there is a multiplicative factor of $k$ in the loss. 
$ ^{(*)}$ Some results can be extended to the agnostic case by standard techniques as discussed in Appendix~\ref{app:online-agn}.}
\vspace{-2em}
\end{table}\label{tb:results}

\vspace{-0.5em}
\subsection{Related work}

Our work is primarily related to the literature on strategic classification (e.g.~\citep{hardt2016strategic,dong2018strategic,chen2020learning,sundaram2021pac,zhang2021incentive,ahmadi2021strategic,ahmadi2023fundamental}). Due to space constraints, we refer the reader to Appendix~\ref{app:related} for a detailed review of related work. 
In the PAC learning, \cite{zhang2021incentive} and \cite{sundaram2021pac} both provided examples of hypothesis classes that are learnable in the standard setting (i.e., the hypothesis class with finite VC dimension) but not learnable in the strategic setting. However, it is unclear what property of manipulation structure leads to such a learnability gap.
\cite{ahmadi2023fundamental} studied regret in our post-manipulation feedback setting. However, they only focused on finite hypothesis classes and thus, did not provide an answer to learnability gaps.
\cite{shao2023strategic,lechner2023strategic} studied the case where the manipulation graph is unknown. However, \cite{lechner2023strategic} focused on a different loss function that penalizes manipulation even when the prediction at the manipulated feature is correct and \cite{shao2023strategic} provided guarantees solely for finite hypothesis classes. 


The literature on strategic learning distinguishes between ``gaming'' (e.g.~\citep{hardt2016strategic, ahmadi2021strategic}) and ``improving'' (e.g.~\citep{kleinberg2020classifiers, haghtalab2020maximizing,ahmadi2022classification}).
In the former, manipulations (such as cheating to pass an exam) do not alter the true qualifications of agents, whereas in the latter, manipulations (such as studying hard) do change these qualifications. Previous literature typically assumes that true qualifications are modeled as a function of feature vectors.
While our work can be classified within the ``gaming'' category (as manipulations do not alter the true qualifications of agents), we adopt a nuanced approach to modeling true qualifications, that is in line with the implications of \cite{ahmadi2023fundamental,shao2023strategic}. Namely, true qualifications are represented as a function of both original feature vectors (which reflect the current state of the agent) and manipulation power (which indicates the ability to change the current state). Consequently, while manipulations do not impact true qualifications, they can encompass both ``gaming'' actions (e.g., students cheating) and ``improving'' actions (e.g., studying hard).

\section{Model}

\subsection{Strategic Classification}
Throughout this work, we focus on the binary classification task in both adversarial online setting and PAC setting.
Let $\cX$ denote the feature vector space, $\cY=\{0,1\}$ denote the label space, and $\cH\subset \cY^\cX$ denote the hypothesis class.
In the strategic setting, an example $(x,y)$, referred to as an \textit{agent} in this context, consists of a pair of feature vector and label.
An agent will try to receive a positive prediction by modifying their feature vector to some reachable feature vectors. 
Following prior works (e.g., ~\cite{zhang2021incentive,lechner2022learning,ahmadi2023fundamental,lechner2023strategic}), we model the set of reachable feature vectors as a manipulation graph $G = (\cX, \cE)$, where the nodes are all feature vectors in $\cX$. 
For any two feature vectors $x, x'\in \cX$, there is a directed arc $(x,x')\in \cE$ from $x$ to $x'$ if and only if an agent with feature vector $x$ can manipulate to $x'$. 
We let $N_G(x) = \{x'|(x,x')\in \cE\}$ denote the out-neighborhood of $x$ in graph $G$ and $N_G[x] = \{x\}\cup N_G(x)$ denote the closed out-neighborhood. 
Let $k(G)$ denote the maximum out-degree of graph $G$. 
When a hypothesis $h\in \cY^\cX$ is implemented, if an agent $(x,y)$ is classified by $h$ as negative at their original feature vector $x$ but has a neighbor which is classified as positive by $h$, she will manipulate to such neighbor.
Otherwise, she will stay at her original feature vector, $x$. Either way, the label of the agent remains~$y$.
Formally,
\begin{definition}[Manipulated feature]
    Given a manipulation graph $G=(\cX,\cE)$ and a hypothesis $h$, the manipulated feature vector of $x$ induced by $(G,h)$ is
    \vspace{-0.5em}
    \begin{equation*}
        \pi_{G,h}(x) =
        \begin{cases}
            v \in \cX_{h,+}\cap N_G(x) & \text{if } h(x) = 0 \text{ and }\cX_{h,+}\cap N_G(x)\neq \emptyset,  \\
            x &\text{otherwise,}
        \end{cases}
    \end{equation*}
    \vspace{-0.5em}
    
    where $\cX_{h,+}=\{x'\in \cX| h(x')=1\}$ is the positive region by $h$, and the tie-breaking here is assumed to be arbitrary unless specified otherwise.
\end{definition}
The post-manipulation labeling of the agent with original feature vector $x$ by a hypothesis $h$ is determined by their prediction at the resulting manipulated feature vector, $\pi_{G,h}(x)$.
\begin{definition}
[Manipulation-induced labeling]
    Given a manipulation graph $G=(\cX,\cE)$ and a hypothesis~$h$, the labeling of $x$ induced by $(G,h)$ is defined as $\bar h_G(x) = h(\pi_{G,h}(x))$.
We define $\bar \cH_{G} = \{\bar h_G|h\in \cH\}$ to be the $(G,\cH)$-induced labeling class of hypothesis class $\cH$.
\end{definition}
When implementing a hypothesis $h$, the loss incurred at the agent $(x,y)$ is the misclassification error when the agent manipulates w.r.t.\ $h$. We define the strategic loss as follows.
\begin{definition}[Strategic loss]\label{def:lstr}
    Given the manipulation graph $G$, for any hypothesis $h\in \cY^\cX$, the strategic loss of $h$ at agent $(x,y)$ is
    \vspace{-0.5em}
    \begin{equation*}
        \lstr_G(h, (x,y)) := \1{\bar h_G(x)\neq y}\,.
    \end{equation*}
    \vspace{-0.5em}
\end{definition}
\vspace{-1em}

%


\subsection{Learning with Strategic Agents}
Let $G^\star$ denote the underlying true manipulation graph and let $k = k(G^\star)$ denote the maximum out-degree of~$G^\star$.
In the strategic setting, a sequence of agents arrives and interacts with the learner.
The learner perceives agents through the lens of the currently implemented hypothesis, and 
lacks the ability to counterfactually inquire about the manipulations previous agents would have adopted under a different hypothesis. 
More formally,
consider a sequence of $T$ agents $(x_1,y_1),\ldots,(x_T,y_T)$ arriving sequentially.
At each round $t$, an agent $(x_t,y_t)$ arrives. Then, the learner implements a hypothesis $h_t\in \cY^{\cX}$ and reveals $h_t$ to the agent.
After observing $h_t$, the agent responds by manipulating her feature vector from $x_t$ to $v_t =\pi_{G^\star,h_t}(x_t)$.
At the end of the round, the learner 
receives the result of her prediction on the manipulated feature, $\hat y_t = h_t(v_t)$,
and the true label, $y_t$. 
The interaction protocol (which repeats for $t=1,\ldots,T$) is described in Protocol~\ref{prot:interaction}.
\begin{protocol}[t]
    \caption{Learner-Agent Interaction}
    \label{prot:interaction}
        \begin{algorithmic}[1]
        \FOR{$t=1,\ldots, T$}
            \STATE The environment picks an agent $(x_t,y_t)$. \textit{// In the online setting, the agent is chosen adversarially, while in the PAC setting, the agent is sampled i.i.d.}
            \STATE The learner implements a hypothesis $h_t\in \cY^\cX$ and discloses $h_t$ to the agent
            \STATE The agent manipulates to $v_t=\pi_{G^\star,h_t}(x_t)$
            \STATE The learner receives prediction $\hat y_t = h_t(v_t)$,  the true label $y_t$ and possibly additional information about $x_t$ and/or $v_t$
        \ENDFOR
        \end{algorithmic}
    \end{protocol}
In this work, we consider both the PAC and adversarial online setting. 
\paragraph{PAC Learning}
In the PAC setting, the agents are sampled from an underlying distribution $\cD$ over $\cX\times \cY$.
For any data distribution $\cD$ over $\cX\times \cY$, the strategic population loss of $h$ is defined as
    \begin{equation*}
        \Lstr_{G^\star, \cD}(h) := \EEs{(x,y)\sim \cD}{\lstr_{G^\star}(h,(x,y))}\,.
    \end{equation*}
Similar to classical PAC learning, the learner's goal is to find a hypothesis $h$ with low strategic population loss $\Lstr_{G^\star,\cD}(h)$.
However, in contrast to classical PAC learning, where the learner obtains a batch of i.i.d. training data, 
the strategic setting involves an interactive training phase. 
The interactive phase is due to the fact that the learner's hypothesis influences the information the learner observes. For this reason, in the interactive phase, the learner sequentially modifies her hypothesis. 
As a result, we explore a scenario in which the learner interacts with $T$ i.i.d. agents (as illustrated in Protocol~\ref{prot:interaction}), and produces a predictor after completing $T$ rounds of interaction.
In the \textit{realizable} case, there exists a perfect classifier $h^\star\in \cH$ with strategic population loss $\Lstr_{G^\star,\cD}(h^\star) = 0$.

\paragraph{Online Learning}
In the online setting, the agents are generated by an adversary. The learner's goal is to minimize the total number of mistakes, $\sum_{t=1}^T\1{\hat y_t \neq y_t}$.
More specifically, the learner's goal is to minimize the Stackelberg regret\footnote{This is 
Stackelberg regret since the learner selects a hypothesis $h_t$ and assumes that the agent will best respond to $h_t$.}  with respect to the best hypothesis $h^\star\in \cH$ in hindsight, had the agents manipulated given $h^\star$:

\vspace{-1em}
\begin{equation*}
    \regret(T) := \sum_{t=1}^T \lstr_{G^\star}(h_t,(x_t,y_t)) - \min_{h^\star\in \cH} \sum_{t=1}^T\lstr_{G^\star}(h^\star,(x_t,y_t))
\end{equation*}
\vspace{-0.5em}

In the \textit{realizable} case, for the sequence of agents $(x_1,y_1),\ldots,(x_T,y_T)$,  there exists a perfect hypothesis $h^\star\in \cH$ that makes no mistakes, i.e., $\lstr_{G^\star}(h^\star,(x_t,y_t)) = 0$ for all $t=1,\ldots, T$.
In this case, the Stackelberg regret is reduced to the number of mistakes.


\subsection{Uncertainties about Manipulation Structure and Features}\label{sec:model-uncertainty}
As discussed in the related work (Appendix~\ref{app:related}), various settings based on information about the manipulation structure and features have been studied. 
In this work, we begin with the simplest fully informative setting, where the underlying manipulation graph $G^\star$ is known, and the original feature vector $x_t$ is revealed to the learner before the implementation of the hypothesis. 
We use this fully informative setting as our baseline model and investigate two forms of uncertainty: either the original feature vector $x_t$ is not disclosed beforehand or the manipulation graph $G^\star$ is unknown.
\begin{enumerate}[topsep = 2pt, itemsep=2pt, parsep=1pt, leftmargin = *]
    \item \textbf{Fully Informative Setting:} Known manipulation graph and known original feature beforehand. More specifically, the true manipulation graph $G^\star$ is known, and $x_t$ is observed before the learner implements~$h_t$.  This setting is the simplest and serves as our baseline, providing the learner with the most information.
    Some may question the realism of observing $x_t$ beforehand, but there are scenarios in which this is indeed the case. For instance, when customers inquire about mortgage rates and provide their credit score or income as input, or in college admissions where colleges can access sophomore-year grades. 
    In cases where each agent represents a subpopulation and known statistics about that population are available, pre-observation of $x_t$ can be feasible.

    \item \textbf{Post-Manipulation Feedback Setting:} 
    Here the underlying graph $G^\star$ is still known, but the original feature vector $x_t$ is not observable before the learner implements the hypothesis. Instead, either the original feature vector $x_t$ or the manipulated feature vector $v_t$ are revealed afterward (we consider both variants). It is worth noting that, since the learner can compute $v_t= \pi_{G^\star,h_t}(x_t)$ on her own given $x_t$, knowing $x_t$ is more informative than knowing~$v_t$. \cite{ahmadi2023fundamental} introduced this setting in the context of online learning problems and offered preliminary guarantees on Stackelberg regret, which depend on the cardinality of the hypothesis class $\cH$ and the maximum out-degree of graph~$G^\star$.
    
    \item \textbf{Unknown Manipulation Graph Setting:} 
    In real applications, the manipulation graph may not always be known. 
    We therefore
    introduce the unknown manipulation graph setting in which we relax the assumption of a known graph. Instead, we assume that the learner has prior knowledge of a certain graph class $\cG$; we consider both the realizable setting in which $G^\star\in\cG$ and the agnostic setting in which $\cG$ serves as a comparator class.
    When the true graph $G^\star$ is undisclosed, the learner cannot compute $v_t$ from $x_t$. 
    Thus, there are several options depending on when the learner gets the feedback: observing $x_t$ before implementing $h_t$ followed by $v_t$ afterward, observing $(x_t,v_t)$ after the implementation, and observing $x_t$ only (or $v_t$ only) after the implementation, arranged in order of increasing difficulty. In this work, we mainly focus on the most informative option, which involves observing $x_t$ beforehand followed by~$v_t$ afterward. However, we show that in the second easiest setting of observing $(x_t,v_t)$ afterward, no deterministic algorithm can achieve sublinear (in $|\cH|$ and $|\cG|$) mistake bound in the online setting.
\end{enumerate}

Note that the first two settings encompass all possible scenarios regarding the observation timing of $x_t$ and~$v_t$ when $\cG^\star$ is known, given that $v_t$ depends on the implemented hypothesis.
It is important to emphasize that our algorithms do not necessitate knowledge of the entire graph. Instead, we only require specific local manipulation structure information, specifically, knowledge of the out-neighborhood of the currently observed feature.


\section{Highlights of Some Main Results}

\subsection{PAC Learning in Fully Informative and Post-Manipulation Feedback Settings}
When the manipulation graph $G^\star$ is known, after obtaining an i.i.d. sample $S = ((x_1, y_1),\ldots, (x_T, y_T))$, we can obtain a hypothesis $\hat h$ by minimizing the empirical strategic loss,

\vspace{-1.5em}
\begin{align}
\restatableeq{\eqERM}{\hat h =\argmin_{h\in \cH} \sum_{t=1}^T \lstr_{G^\star}(h,(x_t,y_t))\,.    }{eq:erm}
\end{align}
\vspace{-1em}


Though in these two settings we consider the learner knows the entire graph $G^\star$, the learner only needs access to the out-neighborhood $N_{G^\star}(x_t)$ to implement ERM. If only the pair $(v_t, N_{G^\star}(v_t))$ is observed, we can obtain the knowledge of $(x_t, N(x_t))$ by implementing $h_t = \1{\cX}$ so that $v_t =x_t$.

We can then guarantee that $\hat h$ achieves low strategic population loss by VC theory. Hence, all we need to do 
for generalization is to bound the VC dimension of $\bar \cH_{G^\star}$. We establish that $\vcd(\bar \cH_{G^\star})$ can be both upper bounded (up to a logarithmic factor in $\vcd(\cH)$) and lower bounded by $\vcd(\cH) \log k$.
This implies that for any hypothesis class $\cH$ that is PAC learnable, it can also be strategically PAC learnable, provided that the maximum out-degree of $G^\star$ is finite. Achieving this only necessitates a sample size that is $\log k$ times larger and this logarithmic factor is necessary in some cases.
\begin{restatable}{theorem}{vcdinduced}\label{thm:vcd-induced}
For any hypothesis class $\cH$ and graph $G^\star$, we have $\vcd(\bar \cH_{G^\star}) \leq d\log(kd)$ where $d = \vcd(\cH)$. Moreover, for any $d,k>0$, there exists a class  $\cH$ with $\vcd(\cH) = d$ and a graph $G^\star$ with maximum out-degree $k$, such that $\vcd(\bar \cH_{G^\star})\geq d\log k$.
\end{restatable}

\begin{proof}[sketch of lower bound]
Let us start with the simple case of $d=1$. 
Consider $(k+1)\log k$ nodes $\cX =\{x_{i,j}|i = 1,\ldots,\log k, j=0,\ldots, k\}$, where there is an arc from $x_{i,0}$ 
to each of $\{x_{i,j}|j=1,\ldots, k\}$ for all $i\in [\log k]$ in $G^\star$. 
For $j=1,\ldots,k$, let $\text{bin}(j)$ denote the $\log k$-bit-binary representation of $j$. 
Let hypothesis $h_j$ be the hypothesis that labels $x_{i,j}$ by $\text{bin}(j)_i$ and all the other nodes by $0$.
Let $\cH = \{h_j|j\in [k]\}$.
Note that $\vcd(\cH) = 1$. This is because for any two points $x_{i,j}$ and $x_{i',j'}$, if $j\neq j'$, no hypothesis will label them as positive simultaneously; if $j=j'$, they can only be labeled as $(\text{bin}(j)_i,\text{bin}(j)_{i'})$ by $h_j$ and $(0,0)$ by all the other hypotheses.
However, $\bar \cH_{G^\star}$ can shatter $\{x_{1,0},x_{2,0},\ldots,x_{\log k, 0}\}$ since $h_j$'s labeling over  $\{x_{1,0},x_{2,0},\ldots,x_{\log k, 0}\}$ is $\text{bin}(j)$. Hence, we have $\vcd(\bar \cH_{G^\star}) = \log k$. We can then extend the example to the case of $\vcd(\cH) = d$ by making $d$ independent copies of $(\mathcal X,\cH, G^\star)$.
The detailed proof of Theorem~\ref{thm:vcd-induced} appears in Appendix~\ref{app:fullInfo}.
\end{proof}
\vspace{-1.75em}
\begin{remark}
    We have similar results for Littlestone dimension as well. More specifically, that $\ld(\bar \cH_{G^\star}) \leq \ld(\cH)\cdot\log k$ for any $\cH, G^\star$ (Theorem ~\ref{thm:fi-online}). 
    We also show that $\ld(\bar \cH_{G^\star}) \geq \ld(\cH)\cdot\log k$ using the same example we construct for VC dimension (Theorem~\ref{thm:fi-online-lb}). In contrast with the PAC setting where we derive the sample complexity bounds by a combinatorial analysis of the VC dimension, in the online setting we obtain the upper bound using an explicit algorithm. 
\end{remark}

\subsection{PAC Learning in the Unknown Manipulation Graph Setting}\label{sec:highlight-ug-pac}
Without the knowledge of $G^\star$ (either the entire graph or the local structure), we can no longer implement ERM.
For any implemented predictor $h$, if $h(x)=0$ and $N_{G^\star}(x)\cap \cX_{h,+}$ is non-empty, 
we assume that the agent will break ties randomly by manipulating to\footnote{This is not necessary. All we need is a fixed tie-breaking way that is consistent in both training and test time.} $\pi_{G^\star,h}(x)\sim \Unif(N_{G^\star}(x)\cap \cX_{h,+})$.
Since each agent can only manipulate within her neighborhood, we define a graph loss $\lgraph$ for graph $G$ as  the 0-1 loss of neighborhood prediction, 

\vspace{-1em}
\begin{equation*}
    \lgraph (G) := \Pr_{(x,y)\sim \cD}(N_{G^\star}(x)\neq N_{G}(x))\,.
\end{equation*}
\vspace{-1em}

Then we upper bound the strategic loss by the sum of the graph loss of $G$ and the strategic loss under the graph $G$, i.e.,

\vspace{-1em}
\begin{equation}
    \Lstr_{G^\star,\cD}(h) \leq \lgraph (G) + \Lstr_{G,\cD}(h), \text{ for all } G\,.\label{eq:l-decompose}
\end{equation}
\vspace{-1em}

This upper bound implies that if we can find a good approximate graph, we can then solve the problem by running ERM over the strategic loss under the approximate graph.
The main challenge of finding a good approximation of $G^\star$ is that \textbf{we cannot obtain an unbiased estimate of $\lgraph (G)$}.
This is because we cannot observe $N_{G^\star}(x_t)$. The only information we could obtain regarding $G^\star$ is a random neighbor of $x_t$, i.e., $v_t \sim \Unif(N_{G^\star}(x_t)\cap \cX_{h_t,+})$. While one may think of finding a consistent graph whose arc set contains all arcs of $(x_t,v_t)$ in the historical data, such consistency is not sufficient to guarantee a good graph. For example, considering a graph $G$ whose arc set is a superset of $G^\star$'s arc set, then $G$ is consistent but its graph loss could be high.
We tackle this challenge through a regulation term based on the empirical degree of the graph. 

In the realizable case, where there exists a perfect graph and a perfect hypothesis, i.e., there exists $G^\star$ such that $\lgraph(G^\star)=0$ and $\Lstr_{G^\star,\cD}(h^\star)=0$, 
our algorithm is described as follows.
\begin{itemize}[topsep = 2pt, itemsep=3pt, parsep=1pt, leftmargin = *]
    \item \underline{Prediction:} At each round $t$, after observing $x_t$, we implement the hypothesis $h(x)=\1{x\neq x_t}$, which labels all every point as positive except $x_t$. This results in a manipulated feature vector $v_t$ sampled uniformly from $N_{G^\star}(x_t)$. 

    \item \underline{Output:} 
    Let $(\hat G, \hat h)$ be a pair that holds

    \vspace{-2em}
    \begin{align}
    (\hat G, \hat h) &\in \argmin_{(G,h)\in (\cG,\cH)} \sum_{t=1}^T \abs{N_G(x_t)}\notag \\
    \text{s.t.}\quad &\sum_{t=1}^T \1{v_t\notin N_G(x_t)} = 0\,.\quad \label{eq:lossG} \\
    &\sum_{t=1}^T \1{\bar{h}_G(x_t)\neq y_t} = 0\,.\label{eq:lossGH}
\end{align}
\vspace{-1.5em}

Notice that Eq~\eqref{eq:lossG} guarantees $G$ is consistent with all arcs of $(x_t,v_t)$ in the historical data and Eq~\eqref{eq:lossGH} guarantees that $h$ has zero empirical strategic loss when the manipulation graph is $G$. $(\hat G, \hat h)$ is the graph-hypothesis pair that satisfies Eqs~\eqref{eq:lossG} and~\eqref{eq:lossGH} with \textbf{minimal empirical degree}.
Finally, we output $\hat h$.
\end{itemize}
Next, we derive sample complexity bound to guarantee that $\hat h$ will have a small error.

\begin{restatable}{theorem}{ugPacRel}\label{thm:ug-pac-rel}
For any hypothesis class $\cH$ with $\vcd(\cH)=d$, the underlying true graph $G^\star$ with maximum out-degree at most $k$, finite graph class $\cG$ in which all graphs have maximum out-degree at most $k$, any $(\cG,\cH)$-realizable data distribution, and any $\epsilon,\delta\in (0,1)$, with probability at least $1-\delta$ over $S\sim \cD^T$ and $v_{1:T}$ where $T = O(\frac{d \log(kd)\log(1/\epsilon) + \log|\cG|\cdot (\log(1/\epsilon) + k) + k\log(1/\delta)}{\epsilon})$, the output $\hat h$ satisfies $\Lstr_{G^\star,\cD}(\hat h)\leq \epsilon$.
\end{restatable}

\vspace{-0.5em}
\begin{proof}[sketch]
Finding a graph-hypothesis pair satisfying Eq~\eqref{eq:lossGH} would guarantee small $\Lstr_{G,\cD}(h)$ by uniform convergence. 
According to the decomposition of strategic loss in Eq~\eqref{eq:l-decompose}, it is sufficient to show the graph loss $\lgraph(\hat G)$ of $\hat G$ is low. 
Note that satisfying Eq~\eqref{eq:lossG} does not guarantee low graph loss. As aforementioned, a graph $G$ whose arc set is a superset of $G^\star$'s arc set satisfies Eq~\eqref{eq:lossG}, but its graph loss could be high. 
Hence, returning a graph-hypothesis pair 
that satisfies Eqs~\eqref{eq:lossG} and~\eqref{eq:lossGH} is not enough and it is crucial to have empirical degree as a regularize.

Now we show the graph loss $\lgraph(\hat G)$ is small.
First, the 0-1 loss of the neighborhood prediction can be decomposed as

\vspace{-1em}
\begin{equation*}
    \1{N_{G^\star}(x)\neq N_{G}(x)}
    \leq \abs{N_{G}(x)\setminus N_{G^\star}(x)} + \abs{N_{G^\star}(x)\setminus N_{G}(x)}\,.
\end{equation*}
\vspace{-1em}

For any $G$ satisfying Eq~\eqref{eq:lossG}, $G$ will have a small $\EEs{x}{\abs{N_{G^\star}(x)\setminus N_{G}(x)}}$. This is because

\vspace{-1em}
\begin{align*}
    \abs{N_{G^\star}(x)\setminus N_{G}(x)} 
    =\abs{N_{G^\star}(x)}\Pr_{v\sim \Unif(N_{G^\star}(x)) }( v\notin N_{G}(x_t))
    \leq k\Pr_{v\sim \Unif(N_{G^\star}(x)) }( v\notin N_{G}(x))\,.
\end{align*}
\vspace{-1em}

Note that $\1{v_t\notin N_{G}(x_t)}$ is an unbiased estimate of $\Pr_{x, v\sim \Unif(N_{G^\star}(x))}(v\notin N_G(x))$. 
Hence, for any graph $G$ satisfying Eq~\eqref{eq:lossG}, $\Pr_{x, v\sim \Unif(N_{G^\star}(x))}(v\notin N_G(x))$ is small by uniform convergence and thus, $\EEs{x}{\abs{N_{G^\star}(x)\setminus N_{G}(x)}}$ is small.

Then we utilize the ``minimal degree'' to connect $\EEs{x}{\abs{N_{G}(x)\setminus N_{G^\star}(x)}}$ and $\EEs{x}{\abs{N_{G^\star}(x)\setminus N_{G}(x)}}$.
For any graph with expected degree $\EEs{x}{|N_{G}(x)|} \leq \EEs{x}{|N_{G^\star}(x)|}$, by deducting $\EEs{x}{|N_{G^\star}(x)\cap N_{G}(x)|}$ from both sides, we have
\begin{equation}
    \EEs{x}{|N_{G}(x)\setminus N_{G^\star}(x)|}\leq \EEs{x}{|N_{G^\star}(x)\setminus N_{G}(x)|}\,,\label{eq:neighborhood-difference}
\end{equation}
Then by minimizing the empirical degree, we can guarantee $\EEs{x}{|N_{G}(x)\setminus N_{G^\star}(x)|}\lesssim \EEs{x}{|N_{G^\star}(x)\setminus N_{G}(x)|}$. The formal proof can be found in Appendix~\ref{app:unknownGraph}.
\end{proof}
So, our algorithm can find a good hypothesis and needs at most $k\log|\cG|\epsilon^{-1}$ more training data. However, as we show in the next theorem, the $\log|\cG|$ factor is necessary.

\begin{restatable}{theorem}{thmUgPacRelLb}\label{thm:ug-pac-rel-lb}
There exists a hypothesis class $\cH$ with $\vcd(\cH) = 1$ and a graph class $\cG$ in which all graphs have a maximum out-degree of $1$. For any algorithm $\cA$, there exists a data distribution $\cD$ such that for any i.i.d. sample of size $T$, there exists a graph $G \in \cG$ consistent with the data history such that when $T \leq \frac{\log|\cG|}{\epsilon}$, we have $\Lstr_{G,\cD}(\hat{h}) \geq \epsilon$.
\end{restatable}

\textbf{Agnostic case } 
Similar to the realizable case, the challenge lies in the fact that  $\lgraph (G)$ is not estimable. Moreover, in the agnostic case 
there might not exist a consistent graph satisfying Eqs~\eqref{eq:lossG} and \eqref{eq:lossGH} now. Instead, we construct the following alternative loss function as a proxy:

\vspace{-0.5em}
\begin{equation*}
    \lproxy(G) := 2\mathbb{E}_{x}[P_{v\sim \Unif(N_{G^\star})(x)}(v\notin N_G(x))] + \frac{1}{k}\mathbb{E}[|N_G(x)|] - \frac{1}{k}\mathbb{E}[|N_{G^\star}(x)|]\,.
\end{equation*}


Note that in this proxy loss, $2\cdot\1{v_t\notin N_G(x_t)}+\frac{1}{k}N_G(x_t)$ is an unbiased estimate of the first two terms, 
and the third term is a constant. Hence, we can find a graph with low proxy loss by minimizing its empirical value.
We then show that this proxy loss is a good approximation of the graph loss when $k$ is small.
\begin{restatable}{lemma}{lmmproxyLoss}\label{lmm:proxy-loss}
    Suppose that $G^\star$ and all graphs $G\in \cG$ have maximum out-degree at most $k$. Then, we have
    \vspace{-0.5em}
    \begin{equation*}
        \frac{1}{k} \lgraph(G) \leq \lproxy(G) \leq 3 \lgraph(G)\,.
    \end{equation*}
    \vspace{-0.5em}
\end{restatable}
\vspace{-1em}

By minimizing the proxy loss $\lproxy(\cdot)$, we can obtain an approximately optimal graph up to a multiplicative factor of $k$ due to the gap between the proxy loss and the graph loss. See details and proofs in Appendix~\ref{app:unknownGraph}.

\paragraph{Application of the Graph Learning Algorithms in Multi-Label Learning }
Our graph learning algorithms have the potential to be of independent interest in various other learning problems, such as multi-label learning. To illustrate, let us consider scenarios like the recommender system, where we aim to recommend movies to users. In such cases, each user typically has multiple favorite movies, and our objective is to learn this set of favorite movies.

This multi-label learning problem can be effectively modeled as a bipartite graph denoted as $G^\star = (\cX, \cY, \cE^\star)$, where $\cX$ represents the input space (in the context of the recommender system, it represents users), $\cY$ is the label space (in this case, it represents movies), and $\cE^\star$ is a set of arcs. In this graph, the presence of an arc $(x, y) \in \cE^\star$ implies that label $y$ is associated with input $x$ (e.g., user $x$ liking movie $y$). Our primary goal here is to learn this graph, i.e., the arcs $\cE^\star$. More formally, given a marginal data distribution $\cD_\cX$, our goal is to find a graph $\hat G$ with minimum neighborhood prediction loss $\lgraph (G)= \Pr_{x\sim \cD_\cX}(N_{G}(x)\neq N_{G^\star}(x))$. Suppose we are given a graph class $\cG$. Then, our goal is to find the graph $\argmin_{G\in \cG}\lgraph (G)$.

However, in real-world scenarios, the recommender system cannot sample a user along with her favorite movie set.
Instead, at each time $t$, the recommender recommends a set of movies $h_t$ to the user and the user randomly clicks one of the movies in $h_t$ that she likes (i.e., $v_t\in N_{G^\star}(x_t)\cap h_t$). Here we abuse the notation a bit by letting $h_t$ represent the set of positive points labeled by this hypothesis.
This setting perfectly fits into our unknown graph setting and our algorithms can find a good graph.


\subsection{Online Learning}\label{sec:highlight-online}
In this section, we mainly present our algorithmic result in the post-manipulation feedback setting. 
The post-manipulation setting is introduced by \cite{ahmadi2023fundamental}, where they exclusively consider the finite hypothesis class $\cH$ and derive Stackelberg regret bounds based on the cardinality $|\cH|$. In contrast, our work extends to dealing with infinite hypothesis classes and offers algorithms with guarantees dependent on the Littlestone dimension. 

In the post-manipulation feedback setting, the learner observes either $x_t$ or $v_t$ after implementing $h_t$. It might seem plausible that observing $x_t$ provides more information than observing $v_t$ since we can compute $v_t$ given $x_t$. However, we show that there is no gap between these two cases. 
We provide an algorithm based on the ``harder'' case of observing $v_t$ and a lower bound based on the ``easier'' case of observing $x_t$ and show that the bounds are tight.

Our algorithm is based on a reduction from strategic online learning to standard online learning. More specifically, given any standard online learning algorithm $\cA$, we construct a weighted expert set $E$ which is initialized to be $\{\cA\}$  with weight $w_\cA =1$. At each round, we predict $x$ as $1$ iff the weight of experts predicting $x$ as $1$ is higher than $1/2(k+1)$ of the total weight.
Once a mistake is made at the observed node $v_t$, we differentiate between the two types of mistakes.
\begin{itemize}[topsep = 2pt, itemsep=3pt, parsep=1pt, leftmargin = *]
    \item \underline{False positive:} Namely, the  labeling of $x_t$ induced by $h_t$ is $1$ but the true label is $y_t = 0$. 
    Since the true label is $0$, the entire neighborhood of $x_t$ should be labeled as $0$  by the target hypothesis, including the observed feature $v_t$. Hence, we proceed by updating all experts incorrectly predicting $v_t$ as $1$ with the example $(v_t,0)$ and halve their weights. 
    \item \underline{False negative:} Namely, the labeling of $x_t$ induced by $h_t$ is $0$, but the true label is $y_t = 1$.
    Thus, $h^\star$ must label some neighbor of $x_t$, say $x^\star$, by $1$. For any expert $A$ labeling the entire neighborhood by $0$, they make a mistake. But we do not know which neighbor should be labeled as $1$. So we split $A$ into $|N_{G^\star}[x_t]|$ number of experts, each of which is fed by one neighbor $x\in N_{G^\star}[x_t]$ labeled by $1$. Then at least one of the new experts is fed with the correct neighbor $(x^\star,1)$. We also split the weight equally and halve it.
\end{itemize}
The pseudo-code is included in Algorithm~\ref{alg:reduction2online-pmf}.
\begin{algorithm}[t]\caption{Red2Online-PMF: Reduction to online learning 
 in the post-manipulation feedback setting}\label{alg:reduction2online-pmf}
    \begin{algorithmic}[1]
        \STATE \textbf{Input: } a standard online learning algorithm $\cA$, maximum out-degree upper bound $k$
        \STATE \textbf{Initialization: } expert set $E = \{\cA\}$ and weight $w_{\cA}=1$
        \FOR{$t=1,2,\ldots$}
        \STATE \underline{Prediction}: at each point $x$, $h_t(x) =1$ if $\sum_{A \in E: A(x) = 1}w_{A} \geq \frac{\sum_{A\in E} w_{\cA_H}}{2(k+1)}$\label{alg-line:k-fraction-weight}
        \STATE \underline{Update}: \textit{//when we make a mistake at the observed node $v_t$}
        \IF{$y_t = 0$}
        \STATE for all $A\in E$ satisfying $A(v_t) = 1$, feed $A$ with $(v_t,0)$ and update $w_{A} \leftarrow \frac{1}{2} w_{A}$
        \ELSIF{$y_t = 1$}
        \STATE for all $A \in E$ satisfying $\forall x\in N_{G^\star}[v_t], A(x) =0$
        \STATE for all $x\in N_{G^\star}[v_t]$, by feeding $(x,1)$ to $A$, we can obtain a new expert $A(x,1)$
        \STATE remove $A$ from $E$ and add $\{A(x,1)|x\in N_{G^\star}[x_t]\}$ to $E$ with weights $w_{A(x,1)} = \frac{w_{A}}{2\abs{N_{G^\star}[x_t]}}$
        \ENDIF
        \ENDFOR
    \end{algorithmic}
\end{algorithm}

Next, we derive a mistake bound for Algorithm~\ref{alg:reduction2online-pmf} that depends on the mistake bound of the standard online learning algorithm, and by plugging in the Standard Optimal Algorithm (SOA), we derive a mistake bound of $O(k \log k \cdot M)$. 
\begin{restatable}{theorem}{fuOnline}\label{thm:fu-online}
For any hypothesis class $\cH$, graph $G^\star$ with maximum out-degree $k$, and a standard online learning algorithm with mistake bound $M$ for $\cH$, for any realizable sequence, we can achieve mistake bound of $O(k \log k \cdot M)$ by Algorithm~\ref{alg:reduction2online-pmf}.
By letting the SOA algorithm by~\cite{littlestone1988learning} as the input standard online algorithm $\cA$,  Red2Online-PMF(SOA)
makes at most $O(k \log k\cdot \ld(\cH))$ mistakes.
\end{restatable}
Compared to the fully informative setting (in which the dependency on $k$ of the optimal mistake bound is $\Theta(\log k)$), the dependency of the number of mistakes made by our Algorithm~\ref{alg:reduction2online-pmf} on $k$ is $O(k\log k)$. However, as we show in the next theorem, 
no deterministic algorithm can make fewer than $\Omega(\ld(\cH)\cdot k)$ mistakes. Hence, our algorithm is nearly optimal (among all deterministic algorithms) up to a logarithmic factor in $k$.
\begin{restatable}{theorem}{fuOnlineLb}\label{thm:fu-online-lb}
For any $k,d>0$, there exists a graph $G^\star$ with maximum out-degree $k$ and a hypothesis class $\cH$ with $\ld(\cH) = d$ such that for any deterministic algorithm, there exists a realizable sequence for which the algorithm will make at least $d(k-1)$ mistakes.
\end{restatable}

\vspace{-1em}
\begin{remark}
    \cite{ahmadi2023fundamental} leave the gap between the upper bound of $O(k \log(\abs{\cH}))$ and the lower bound of $\Omega(k)$ as an open question. Theorems~\ref{thm:fu-online} and \ref{thm:fu-online-lb} closes this gap up to a logarithmic factor in $k$.
\end{remark}

\textbf{Extension to Unknown Graph Setting }
In the unknown graph setting, we do not observe $N_{G^\star}(x_t)$ or $N_{G^\star}(v_t)$ anymore.
We then design an algorithm by running an instance of Algorithm~\ref{alg:reduction2online-pmf} over the majority vote of the graphs. 
More specifically, after observing $x_t$,  let $\tilde N(x_t)$ denote the set of nodes $x$ satisfying that $(x_t,x)$ is an arc in at least half of the consistent graphs. 
\begin{itemize}[topsep = 2pt, itemsep=3pt, parsep=1pt, leftmargin = *]
    \item \underline{Reduce inconsistent graphs}: For any $x\notin \tilde N(x_t)$, which means $(x_t,x)$ is an arc in at most half of the current consistent graphs, we predict $h_t(x) =1$. If we make a mistake at such a $v_t = x$, it implies that $(x_t,x)$ is indeed an arc in the true graph, so we can eliminate at least half of the graphs.
    \item \underline{Approximate $N_{G^\star}(x_t)$ by $\tilde N(x_t)$}: For any $x\in \tilde N(x_t)$, we predict $h_t(x)$ by following Algorithm~\ref{alg:reduction2online-pmf}. 
    We update Algorithm~\ref{alg:reduction2online-pmf} in the same way when we make a false positive mistake at $v_t$ as it does not require the neighborhood information. However, when we make a false negative mistake at $v_t$, it implies that $v_t =x_t$ and the entire neighborhood $N_{G^\star}(x_t)$ is labeled as $0$. We then utilize $\tilde N(x_t)$ as the neighborhood feedback to update  Algorithm~\ref{alg:reduction2online-pmf}.
    Note that this majority vote neighborhood $\tilde N(x_t)$ must be a superset of $N_{G^\star}(v_t)=N_{G^\star}(x_t)$ since the entire $N_{G^\star}(v_t)$ is labeled as $0$ by $h_t$ and every $x\notin \tilde N(x_t)$ is labeled as $1$. Moreover, the size of $\tilde N(x_t)$ will not be greater than $2k$ since all graphs have maximum degree at most $k$.
\end{itemize}
We defer the formal description of this algorithm to Algorithm~\ref{alg:ug-online} that appears in Appendix~\ref{app:unknownGraph}. In the following, we derive an upper bound over the number of mistakes for this algorithm.
\begin{restatable}{theorem}{ugOnline}\label{thm:ug-online}
    For any hypothesis class $\cH$, the underlying true graph $G^\star$ with maximum out-degree at most $k$, finite graph class $\cG$ in which all graphs have maximum out-degree at most $k$, for any realizable sequence, Algorithm~\ref{alg:ug-online} will make at most $O(k\log(k)\cdot \ld(\cH)+ \log |\cG|)$ mistakes.   
\end{restatable}
So when we have no access to the neighborhood information, we will suffer $\log |\cG|$ more mistakes. However, we show that this $\log |\cG|$ term is necessary for all deterministic algorithms.

\begin{restatable}{theorem}{ugOnlineLb}
    \label{thm:ug-online-lb}
    For any $n\in \NN$, there exists a hypothesis class $\cH$ with $\ld(\cH) =1$, a graph class $\cG$ satisfying that all graphs in $\cG$ have maximum out-degree at most $1$ and $|\cG| = n$ such that for any deterministic algorithm, there exists a graph $G^\star\in \cG$ and a realizable sequence for which the algorithm will make at least $\frac{\log n}{\loglog n}$ mistakes.  
\end{restatable}
Finally, as mentioned in Section~\ref{sec:model-uncertainty}, there are several options for the feedback, including observing $x_t$ beforehand followed by $v_t$ afterward, observing $(x_t,v_t)$ afterward, and observing either $x_t$ or $v_t$ afterward, arranged in order of increasing difficulty. We mainly focus on the simplest one, that is, observing $x_t$ beforehand followed by $v_t$ afterward. We show that even in the second simplest case of observing $(x_t,v_t)$ afterward, any deterministic algorithm will suffer mistake bound linear in $|\cG|$ and $|\cH|$.
\begin{restatable}{proposition}{uglessinfo}\label{prop:online-v-lb}
    When $(x_t,v_t)$ is observed afterward, for any $n\in \NN$, there exists a class $\cG$ of graphs of degree $2$ and a hypothesis class $\cH$ with $|\cG| = |\cH| = n$ such that for any deterministic algorithm, there exists a graph $G^\star\in \cG$ and a realizable sequence for which the algorithm will make at least $n-1$ mistakes.
\end{restatable}

\section{Discussion}
In this work, we have investigated the learnability gaps of strategic classification in both PAC and online settings. We demonstrate that learnability implies strategic learnability when the manipulation graph has a finite maximum out-degree, $k<\infty$. Additionally, strategic manipulation does indeed render learning more challenging.
In scenarios where we consider both the true graph information and pre-manipulation feedback, manipulation only results in an increase in both sample complexity and regret by a $\log k$ multiplicative factor. However, in cases where we only have post-manipulation feedback, the dependence on $k$ in the sample complexity remains $\log k$, but increases to $k$ in the regret.
When the true graph is unknown and only a graph class $\cG$ is available, there is an additional increase of a $\log|\cG|$ additive factor in both sample complexity and regret. Our algorithms for learning an unknown graph are of independent interest to the multi-label learning problem.

Several questions remain open. The first pertains to agnostic online learning. We have explored extending some realizable online learning results to the agnostic case in Appendix~\ref{app:online-agn} through a standard reduction technique. However, it is unclear how to address this issue, especially in the case where there is no perfect graph in the class.
Additionally, there are several other open technical questions, such as improving the lower bounds presented in the work and eliminating the multiplicative factor in Theorem~\ref{thm:ug-pac-agn}.

\section*{Acknowledgements}
Lee Cohen is supported by the Simons Foundation Collaboration on the Theory of Algorithmic Fairness, the Sloan Foundation Grant 2020-13941, and the Simons Foundation investigators award 689988.

Yishay Mansour was supported by funding from the European Research Council (ERC) under the European Union’s Horizon 2020 research and innovation program (grant agreement No. 882396), by the Israel Science Foundation,  the Yandex Initiative for Machine Learning at Tel Aviv University and a grant from the Tel Aviv University Center for AI and Data Science (TAD).

Shay Moran is a Robert J.\ Shillman Fellow; he acknowledges support by ISF grant 1225/20, by BSF grant 2018385, by an Azrieli Faculty Fellowship, by Israel PBC-VATAT, by the Technion Center for Machine Learning and Intelligent Systems (MLIS), and by the the European Union (ERC, GENERALIZATION, 101039692). Views and opinions expressed are however those of the author(s) only and do not necessarily reflect those of the European Union or the European Research Council Executive Agency. Neither the European Union nor the granting authority can be held responsible for them.

Han Shao was supported in part by the National Science Foundation under grants CCF-2212968 and ECCS-2216899, by the Simons Foundation under the Simons Collaboration on the Theory of Algorithmic Fairness, and by the Defense Advanced Research Projects Agency under cooperative agreement HR00112020003. The views expressed in this work do not necessarily reflect the position or the policy of the Government and no official endorsement should be inferred.


\bibliographystyle{abbrvnat}
\bibliography{ref}

\newpage
\appendix

\section{Related Work}\label{app:related}
Our work aligns with strategic classification, a recent line of research that explores the influence of strategic behavior on decision making and machine learning algorithms. This line of study is extensive and encompasses various settings, including distributional vs. online, different knowledge of manipulation structure, pre-manipulation vs. post-manipulation feedback, and the goals of gaming vs. improving. However, none of the prior work has addressed the learnability gaps between standard learning and strategic learning studied in this work.

Strategic classification was first studied in a distributional model by~\cite{hardt2016strategic}, assuming that agents manipulate by best responding with respect to a uniform cost function known to the learner. Building on the framework of \cite{hardt2016strategic}, \cite{lechner2022learning,hu2019disparate,milli2019social,sundaram2021pac,zhang2021incentive} studied the distributional learning problem with the goal of returning a classifier with small strategic population loss, and all of them assumed that the manipulations are predefined and known to the learner, either by a cost function or a predefined manipulation graph. Then \cite{lechner2023strategic} and \cite{shao2023strategic} studied the unknown manipulation structure. However, \cite{lechner2023strategic} considered a different loss function that penalizes manipulation even when the prediction at the manipulated feature is correct, and \cite{shao2023strategic} provided guarantees only for finite hypothesis classes.

For the online setting,
\cite{chen2020learning, dong2018strategic,ahmadi2021strategic,ahmadi2023fundamental,shao2023strategic} studied Stackelberg regret in the online setting learning. 
However, \cite{dong2018strategic} studied Stackelberg regret in linear classification and focused on finding appropriate conditions of the cost function to achieve sub-linear Stackelberg regret.
\cite{chen2020learning} studied linear classification under uniform ball manipulations and  \cite{ahmadi2021strategic,ahmadi2023fundamental} studied Stackelberg regret under a predefined and known manipulation. \cite{shao2023strategic} studied the mistake bound under unknown manipulation structures, but only provided results for finite hypothesis classes.

\paragraph{Models of manipulation structure} \cite{hardt2016strategic} modeled manipulation structure by a cost function, where agents manipulate to maximize their utility function, which is the reward of obtaining a positive prediction minus the cost of manipulation.
It then naturally leads to two modeling ways of manipulation structure: The first is a geometric model (e.g., ~\cite{dong2018strategic,sundaram2021pac,chen2020learning, Ghalme21,haghtalab2020maximizing,ahmadi2021strategic,shao2023strategic}), where each agent is able to manipulate within a bounded-radius ball in some known metric space. Other works (e.g., \cite{zhang2021incentive,lechner2022learning,lechner2023strategic,ahmadi2023fundamental}) modeled feasible manipulations using a manipulation graph, where each agent is represented by a node $x$, and an arc from $x$ to $x'$ indicates that an agent with the feature vector $x$ may adjust their features to $x'$ if it increases their likelihood of receiving a positive classification. However, these two ways of modeling are actually equivalent as we can define a manipulation graph given the manipulation balls and we can define ball manipulations w.r.t. the graph distance on a predefined known graph. 

\paragraph{Different information settings}
All PAC learning scenarios involving either known manipulation structures or known local neighborhoods \citep{hardt2016strategic,lechner2022learning,sundaram2021pac,zhang2021incentive,hu2019disparate,milli2019social} can be encompassed within our fully informative setting. We remark that the work of \cite{lechner2022learning} considered the task of learning the manipulation graph, but assumed that each sample is of the form $(x,$ neighborhood of $x)$, which is closer to our fully informative setting as our results also only require the neighborhood information.
\cite{ahmadi2021strategic} examined online learning within our fully informative setting, while \cite{ahmadi2023fundamental} explored online learning within our post-manipulation feedback setting.
Regarding unknown graph settings, \cite{lechner2023strategic} and \cite{shao2023strategic} have made contributions. However, \cite{lechner2023strategic} considered a different loss function that penalizes manipulation even when the prediction at the manipulated feature is correct and \cite{shao2023strategic} provided guarantees solely based on the cardinality of the hypothesis class.
We primarily study strategic classification in online and distributional settings in various restrictive feedback settings. 


\paragraph{Other works.} Other works explored strategic classification in various contexts, e.g.: noisy classifiers~\citep{Braverman20}, causality~\citep{miller2020strategic,shavit2020causal}, noisy agents~\citep{jagadeesan2021alternative}, decision making~\citep{Khajehnejad19,tang2021linear}, sequential setting~\citep{cohen23}, social objectives (e.g., social burden)~\citep{hu2019disparate, Liu20,milli2019social,lechner2023strategic}, Strategic interactions in the regression~\citep{bechavod2021gaming}, non-myopic agents~\citep{haghtalab2022learning,harris2021stateful}, and lack of transparency in decision rules~\citep{bechavod2022information}.
Beyond strategic classification, there is a more general research area of learning using data from strategic sources, such as a single data generation player who manipulates the data distribution~\citep{bruckner2011stackelberg,dalvi2004adversarial}. Adversarial perturbations can be viewed as another type of strategic source~\citep{montasser2019vc}.

\section{Learning Complexity- Background}\label{app:background}
Here we recall the definitions of the Vapnik–Chervonenkis dimension  ($\vcd$)~\citep{vapnik:74}, and Littlestone dimension ($\ld$) \citep{littlestone1988learning}. 

\begin{definition}[VC-dimension]\label{def:vc}
Let $\mathcal{\mathcal{H}} \subseteq \cY^\mathcal{X}$ be a hypothesis class. A subset $S = \{x_1, ..., x_{|S|}\} \subseteq \mathcal{X}$ is shattered by $\mathcal{H}$ if:
$\left| \left\{
\left(h(x_1), ..., h(x_{|S|})\right) : h \in \mathcal{H}
\right\} \right|
= 2^{|S|}$.
The VC-dimension of $\mathcal{H}$, denoted $VCdim(\mathcal{H})$, is the maximal cardinality of a subset $S \subseteq \mathcal{X}$ shattered by $\mathcal{H}$.
\end{definition}

\begin{definition}[$\cH$-shattered tree]
    Let $\cH \subseteq \cY^\cX$ be a hypothesis class and let $d\in\mathbb N$. A sequence $(x_1,\ldots,x_{2^d-1}) \in \X^{2^d-1}$ is an $\cH$\textup{-shattered tree} of depth $d$ if, for every labeling $(y_1,\ldots,y_d)\in\cY^d$, there exists $h\in\cH$ such that for all $i \in [d]$ we have that $h(x_{j_i}) = y_i$, where $j_i = 2^{i-1} + \sum_{k=1}^{i-1} y_k 2^{i-1-k}$. 
\end{definition}

\begin{definition}[Littlestone dimension] 
The \textup{Littlestone dimension} of a hypothesis class $\cH$ is the maximal
depth of a tree shattered by $\cH$. 
\end{definition}

\section{Fully Informative Setting}\label{app:fullInfo}
In this setting, the manipulation graph $G^\star$ is known and the learner observes $x_t$ before implementing $h_t$. We remark that though the learner has knowledge of the entire graph $G^\star$ in this model, the algorithms in this section only require access to the out-neighborhood $N_{G^\star}(x_t)$ of $x_t$.
\subsection{PAC Learning}
In the PAC setting, no matter what learner $h_t$ implemented in each round, we obtain $T$ i.i.d. examples $S = ((x_1, y_1),\ldots, (x_T, y_T))$ after $T$ rounds of interaction.
Then, by running ERM over $\cH$ w.r.t. the strategic loss, we can obtain a hypothesis $\hat h$ by minimizing the empirical strategic loss,
\begin{align*}
\eqERM
\end{align*}
Since $\lstr_{G^\star}(h,(x,y))$ (see Def~\ref{def:lstr}) only depends on the graph $G^\star$ through the neighborhood $N_{G^\star}(x)$, we can optimize Eq~\eqref{eq:erm} given only $\{N_{G^\star}(x_t)|t\in [T]\}$ instead of the entire graph $G^\star$.
Readers familiar with VC theory might notice that the sample complexity can be bounded by the VC dimension of the $(G^\star,\cH)$-induced labeling class $\bar \cH_{G^\star}$, $\vcd(\bar \cH_{G^\star})$.  This is correct and has been shown in~\cite{sundaram2021pac}. However, since our goal is to connect the PAC learnability in the strategic setting and standard setting, we provide sample complexity guarantees dependent on the VC dimension of $\cH$. In fact, our sample complexity results are proved by providing upper and lower bounds of $\vcd(\bar \cH_{G^\star})$ using $\vcd(\cH)$.

\vcdinduced*
\begin{proof}
    For any $x\in \cX$, if two hypotheses $h, h'$ label its closed neighborhood $N_{G^\star}[x] = \{x\}\cup N_{G^\star}(x)$ in the same way, i.e., $h(z)=h'(z), \forall z\in N_{G^\star}[x]$, then the labeling of $x$ induced by $(G^\star,h)$ and $(G^\star,h')$ are the same, i.e., $\bar{h}_{G^\star}(x) = \bar{h'}_{G^\star}(x)$. 
Hence, for any $n$ points $X = \{x_1,x_2,\ldots,x_n\}\subset \cX$, the labeling of $X$ induced by $(G^\star,h)$ is determined by the prediction of $h$ at the their closed neighborhoods, $h(\cup_{x\in X}N_{G^\star}[x])$. 
Since $\abs{\cup_{x\in X}N_{G^\star}[x]}\leq n(k+1)$, there are at most $O((nk)^d)$ different ways of implementation since $\vcd(\cH) = d$ by Sauer-Shelah-Perels Lemma.
If $n$ points can be shattered by $\bar \cH_{G^\star}$, then the number of their realized induced labeling ways is $2^n$, which should be bounded by the number of possible implementations of their neighborhoods, which is upper bounded by $O((nk)^d)$.
Hence, $2^n\leq O((nk)^d)$ and thus, we have $n\leq O(d\log(kd))$. Hence, we have $\vcd(\bar \cH_{G^\star}) \leq O(d\log(kd))$.

For the lower bound, let us start with the simplest case of $d=1$. Consider $(k+1)\log k$ nodes $\{x_{i,j}|i = 1,\ldots,\log k, j=0,\ldots, k\}$, where $x_{i,0}$ is connected to $\{x_{i,j}|j=1,\ldots, k\}$ for all $i\in [\log k]$. 
For $j=1,\ldots,k$, let $\text{bin}(j)$ denote the $\log k$-bit-binary representation of $j$. 
Let hypothesis $h_j$ be the hypothesis that labels $x_{i,j}$ by $\text{bin}(j)_i$ and all the other nodes by $0$.
Let $\cH = \{h_j|j\in [k]\}$.
Note that $\vcd(\cH) = 1$. This is because for any two points $x_{i,j}$ and $x_{i',j'}$, if $j\neq j'$, no hypothesis will label them as positive simultaneously; if $j=j'$, they can only be labeled as $(\text{bin}(j)_i,\text{bin}(j)_{i'})$ by $h_j$ and $(0,0)$ by all the other hypotheses.
However, $\bar \cH_{G^\star}$ can shatter $\{x_{1,0},x_{2,0},\ldots,x_{\log k, 0}\}$ since $h_j$'s labeling over  $\{x_{1,0},x_{2,0},\ldots,x_{\log k, 0}\}$ is $\text{bin}(j)$.
We can extend the example to the case of $\vcd{\cH} = d$ by making $d$ independent copies of $(\mathcal X,\cH, G^\star)$.
\end{proof}

Following directly Theorem~\ref{thm:vcd-induced} through the tools from VC theory~\citep{vapnik:71,vapnik:74}, we prove the following sample complexity upper and lower bounds.

\begin{restatable}{theorem}{thmFiPacRel}\label{thm:fi-pac}
    For any hypothesis class $\cH$ with $\vcd(\cH)=d$, graph $G^\star$ with maximum out-degree $k$, any data distribution $\cD$ and any $\epsilon,\delta\in (0,1)$, with probability at least $1-\delta$ over $S\sim \cD^T$ where $T = O(\frac{d\log(kd)+\log(1/\delta)}{\epsilon^2})$, the output $\hat h$ satisfies $\Lstr_{G^\star,\cD}(\hat h)\leq \min_{h^\star\in \cH}\Lstr_{G^\star,\cD}(h^\star)+ \epsilon$.

    In the realizable case, the output $\hat h$ satisfies $\Lstr_{G^\star,\cD}(\hat h)\leq \epsilon$ when $T = O(\frac{d\log(kd)\log(1/\epsilon)+\log(1/\delta)}{\epsilon})$.
\end{restatable}




\begin{restatable}{theorem}{fiPacLb}\label{thm:fi-pac-lb}
    For any $k,d>0$, there exists a graph $G^\star$ with maximum out-degree $k$ and a hypothesis class $\cH$ with VC dimension $d$ such that for any algorithm, there exists a distribution $\cD$ for which  achieving $\Lstr_{G^\star,\cD}(\hat h)\leq \min_{h^\star\in \cH}\Lstr_{G^\star,\cD}(h^\star)+ \epsilon$ with probability $1-\delta$ requires sample complexity $\Omega(\frac{d\log k + \log(1/\delta)}{\epsilon^2})$. In the realizable case, it requires sample complexity $\Omega(\frac{d\log k + \log(1/\delta)}{\epsilon})$.
\end{restatable}



\subsection{Online Learning}
Recall that in the online setting, the agents are generated by an adversary, and the learner’s goal is to minimize
the Stackelberg regret (which is the mistake bound in the realizable case). 
Similar to Algorithm~\ref{alg:reduction2online-pmf}, we will show that any learning algorithm in the standard online setting can be converted to an algorithm in the strategic online setting, with the mistake bound increased by a logarithmic multiplicative factor of $k$. The main difference from Algorithm~\ref{alg:reduction2online-pmf} is that now we know $x_t$ before implementing the hypothesis, we do not need to reduce the weight of predicting positive by a factor $1/k$ as we did in line~\ref{alg-line:k-fraction-weight} of Algorithm~\ref{alg:reduction2online-pmf}.

\begin{algorithm}[t]\caption{Red2Online-FI: Reduction to online learning 
 in the fully informative setting}\label{alg:reduction2online-fi}
    \begin{algorithmic}[1]
        \STATE \textbf{Input: } a standard online learning algorithm $\cA$
        \STATE \textbf{Initialization: } let the expert set $E = \{\cA\}$ and weight $w_{\cA}$ =1
        \FOR{$t=1,2,\ldots$}
        \STATE \underline{Prediction}: after observing $x_t$
        \IF{$\sum_{A\in E: \exists x\in N_{G^\star}[x_t], A(x) = 1}w_{A} \geq {\sum_{A\in E} w_{A}}/{2}$} \STATE let $h_t(x_t) =1$ and $h_t(x)=0$ for all $x\neq x_t$
        \ELSE 
        \STATE let $h_t =\ind{\emptyset}$ be the all-negative function
        \ENDIF
        \STATE \underline{Update}:  when we make a mistake
        \IF{$y_t = 0$}
        \STATE for all $A\in E$ satisfying $\exists x\in N_{G^\star}[x_t], A(x) = 1$, feed $A$ with $(x,0)$ and update $w_{A} \leftarrow \frac{1}{2} w_{A}$
        \ELSIF{$y_t = 1$}
        \STATE for all $A \in E$ satisfying $\forall x\in N_{G^\star}[x_t], A(x) =0$
        \STATE for all $x\in N_{G^\star}[x_t]$, by feeding $(x,1)$ to $A$, we can obtain a new expert $A(x,1)$
        \STATE remove $A$ from $E$ and add $\{A(x,1)|x\in N_{G^\star}[x_t]\}$ to $E$ with weights $w_{
        A(x,1)} = \frac{w_{A}}{2\abs{N_{G^\star}[x_t]}}$
        \ENDIF
        \ENDFOR
    \end{algorithmic}
\end{algorithm}
\begin{restatable}{theorem}{fiOnline}\label{thm:fi-online}
For any hypothesis class $\cH$, graph $G^\star$ with maximum out-degree $k$, and a standard online learning algorithm with mistake bound $M$ for $\cH$, for any realizable sequence, we can achieve mistake bound of $O(M \log k)$ by Algorithm~\ref{alg:reduction2online-fi}.
By letting the SOA algorithm by~\cite{littlestone1988learning} as the input standard online algorithm $\cA$,  Red2Online-FI(SOA)
makes at most $O(\ld(\cH)\log k)$ mistakes.
\end{restatable}
\begin{proof}
Let $\cA$ be an online learning algorithm with mistake bound $M$.
Let $W_t$ denote the total weight of experts at the beginning of round $t$. When we make a mistake at round $t$, there are two cases.
\begin{itemize}
    \item False positive: It means that the induced labeling of $x_t$ is $1$ but the true label is $y_t = 0$. By our prediction rule, we have $\sum_{\cA_H\in E: \exists x\in N_{G^\star}[x_t], \cA_H(x) = 1}w_{\cA_H} \geq \frac{W_t}{2}$, and thus, we will reduce the total weight by at least $\frac{W_t}{4}$. Hence, we have $W_{t+1}\leq (1-\frac{1}{4})W_t$.
    \item False negative: It means that the induced labeling of $x_t$ is $0$, but the true label is $y_t = 1$. This implies that our learner $h_t$ labeled the entire neighborhood of $x_t$ with $0$, but $h^\star$ labels some point in the neighborhood by $1$.
By our prediction rule, we have $\sum_{\cA_H\in E:\forall x\in N_{G^\star}[x_t]\cA_H(x) = 0}w_{\cA_H} \geq \frac{W_t}{2}$.
We split $\cA_H$ into $\{\cA_{H_{x,+1}}|x\in N_{G^\star}[x_t]\}$ and split the weight $w_{\cA_H}$ equally and multiply by $\frac{1}{2}$. Thus, we have $W_{t+1} \leq (1-\frac{1}{4})W_t$.
\end{itemize}
Hence, whenever we make a mistake the total weight is reduced by a multiplicative factor of $\frac{1}{4}$. Let $N$ denote the total number of mistakes by Red2Online-FI($\cA$). Then we have total weight $\leq (1-1/4)^N$.

On the other hand, note that there is always an expert, whose hypothesis class contains $h^\star$, being fed with examples $(\pi_{G^\star,h^\star}(x_t),y_t)$.
At each mistake round $t$, if the expert (containing $h^\star$) makes a false positive mistake, its weight is reduced by half.
If the expert made a false negative mistake, it is split into a few experts, one of which is fed by $(\pi_{G^\star,h^\star}(x_t),y_t)$. This specific new expert will contain $h^\star$ and the weight is reduced by at most $\frac{1}{2(k+1)}$.
Thus, since this expert will make at most $M$ mistakes, the weight of such an expert will be at least $(\frac{1}{2(k+1)})^M$.

Thus we have $(1-1/4)^N \geq  1/(2(k+1))^M$, which yields the bound $N\leq 4M\ln(2(k+1))$.
\end{proof}

Next, we derive a lower bound for the number of mistakes. 
\begin{restatable}{theorem}{fiOnlineLb}\label{thm:fi-online-lb}
    For any $k,d>0$, there exists a graph $G^\star$ with maximum out-degree $k$ and a hypothesis class $\cH$ with Littlestone dimension $\ld(\cH) = d$ such that for any algorithm, there exists a realizable sequence for which the algorithm will make at least $d\log k$ mistakes.
\end{restatable}

\begin{proof}
    Consider the same hypothesis class $\cH$ and the graph $G^\star$ as Theorem~\ref{thm:fi-pac-lb}. Again, let us start with the simplest case of $d=1$. Consider $(k+1)\log k$ nodes $\{x_{i,j}|i = 1,\ldots,\log k, j=0,\ldots, k\}$, where $x_{i,0}$ is connected to $\{x_{i,j}|j=1,\ldots, k\}$ for all $i\in [\log k]$. 
For $j=1,\ldots,k$, let $\text{bin}(j)$ denote the $\log k$-bit-binary representation of $j$. 
Let hypothesis $h_j$ labels $x_{i,j}$ by $\text{bin}(j)_i$ and all the other nodes by $0$.
Let $\cH = \{h_j|j\in [k]\}$. The Littlestone dimension $\ld(\cH) = 1$ because there is an easy online learning in the standard setting by predicting the all-zero function until we observe a positive point $x_{i,j}$. Then we know $h_j$ is the target hypothesis.
    
At round $t$, the adversary picks $x_{t,0}$. For any deterministic learning algorithm:
\begin{itemize}
    \item If the learner predicts any point in $\{x_{t,j}|j=0,1,\ldots,k\}$ by positive, then the prediction is positive. We know half of the hypotheses predict negative. Hence the adversary forces a mistake by letting $y_t = 0$ and reduces half of the hypotheses.
    \item If the learner predicts all-zero over $\{x_{t,j}|j=0,1,\ldots,k\}$, the adversary forces a mistake by letting $y_t = 1$ and reduces half of the hypotheses.
\end{itemize}
We can extend the example to the case of $\ld{\cH} = d$ by making $d$ independent copies of $(X,\cH, G^\star)$. Then we prove the lower bound for all deterministic algorithms.
We can then extend the lower bound to randomized algorithms by the technique of \cite{FilmusHMM23}.
\end{proof}
\section{Post-Manipulation Feedback Setting}\label{app:postMani}
In this setting, the underlying graph $G^\star$ is known, but the original feature $x_t$ is not observable before the learner selects the learner. Instead, either the original feature $x_t$ or the manipulated feature $v_t$ is revealed afterward.
For PAC learning, since we can obtain an i.i.d. sample $(x_t,N_{G^\star}(x_t),y_t)$ by implementing $h_t = \1{\cX}$, we can still run ERM and obtain $\hat h$ by optimizing Eq~\eqref{eq:erm}. Hence, the positive PAC learning result (Theorem~\ref{thm:fi-pac}) in the fully informative setting should still apply here. Since the post-manipulation feedback setting is harder than the fully informative setting, the negative result (Theorem~\ref{thm:fi-pac-lb}) also holds. However, this is not true for online learning.

\subsection{Online Learning}
We have stated the algorithms and results in Section~\ref{sec:highlight-online}. 
Here we only provide the proofs for the theorems.

\fuOnline*
\begin{proof}
Given any online learning algorithm $\cA$ with mistake bound $M$.
Let $W_t$ denote the total weight of experts at the beginning of round $t$. When we make a mistake at round $t$, there are two cases.
\begin{itemize}
    \item False positive: It means that the induced labeling of $x_t$ is $1$ but the true label is $y_t = 0$. Since the true label is $0$, the neighbor $v_t$ we observe is labeled $0$ by the target hypothesis $h^\star$. In this case, we proceed by updating all experts that predict $v_t$ with $1$ with the example $(v_t,0)$ and we also half their weights. Since $h_t(v_t) = 1$, we have $\sum_{\cA_H \in E: \cA_H(x) = 1}w_{\cA_H} \geq \frac{W_t}{2(k+1)}$. Therefore, we have $W_{t+1}\leq W_t(1- \frac{1}{4(k+1)})$.

    \item False negative: It means that the induced labeling of $x_t$ is $0$, but the true label is $y_t = 1$. This implies that our learner $h_t$ labeled the entire neighborhood of $x_t = v_t$ with $0$, but $h^\star$ labels some point in the neighborhood by $1$.
By our prediction rule, we have 
\begin{align*}
    \sum_{\cA_H\in E:\forall x\in N_{G^\star}[x_t]\cA_H(x) = 0}w_{\cA_H}
    =&  W_t - \sum_{x\in N_{G^\star}[x_t]} \sum_{\cA_H\in E: \cA_H(x) = 1} w_{\cA_H}
    \geq  (1- (k+1)\frac{1}{2(k+1)})W_t \\
    =&\frac{1}{2}W_t
\end{align*}
We split $\cA_H$ into $\{\cA_{H_{x,+1}}|x\in N_{G^\star}[x_t]\}$ and split the weight $w_{\cA_H}$ equally and multiply by $\frac{1}{2}$. Thus, we have $W_{t+1} \leq (1-\frac{1}{4})W_t$.
\end{itemize}

Hence, whenever we make a mistake the total weight is reduced by a multiplicative factor of $\frac{1}{4(k+1)}$. Let $N$ denote the total number of mistakes by Red2Online($\cA$)-FI. Then we have total weight $\leq (1-\frac{1}{4(k+1)})^N$.

On the other hand, note that there is always an expert, whose hypothesis class contains $h^\star$, being fed with examples $(\pi_{G^\star,h^\star}(x_t),y_t)$.
At each mistake round $t$, if the expert (containing $h^\star$) makes a false positive mistake, its weight is reduced by half.
If the expert made a false negative mistake, it is split into a few experts, one of which is fed by $(\pi_{G^\star,h^\star}(x_t),y_t)$. This specific new expert will contain $h^\star$ and the weight is reduced by at most $\frac{1}{2(k+1)}$.
Thus, since this expert will make at most $M$ mistakes, the weight of such an expert will be at least $(\frac{1}{2(k+1)})^M$.

Thus we have $(1-\frac{1}{4(k+1)})^N \geq  1/(2(k+1))^M$, which yields the bound $N\leq 4kM\ln(2(k+1))$.
\end{proof}

\fuOnlineLb*        
\begin{proof}
    Consider a star graph $G = (X,E)$ where $x_0$ is the center and $x_1,\ldots,x_k$ are the leaves. 
    The hypothesis class is singletons over leaves, i.e., $H = \{h_i|i\in [k]\}$ with $h_i = \1{x_i}$. Similar to the proof of Theorem~4.6 in  \cite{ahmadi2023fundamental}, any deterministic algorithm will make at least $k-1$ mistakes.
    At round $t$:
\begin{itemize}
    \item If the learner predicts all points in $X$ as $0$, the adversary picks the agent $x_{0}$. The agent will not move, the induced labeling is $0$, and the true label is $1$ (no matter which singleton is the target hypothesis). We do not learn anything about the target hypothesis.
    \item If the learner predicts $x_{0}$ as positive, then the adversary picks the agent  $x_{0}$. The agent does not move, the induced labeling is $1$, and the true label is $0$ (no matter which singleton is the target hypothesis). Again, we learn nothing about the target hypothesis.
    \item If the learner predicts any node $x_{i}$ with $i$ as positive, the adversary picks the agent $x_{i}$ and label it by $0$. The learner's induced labeling of $x_i$ is $1$, and the true label is $0$. Only one hypothesis $\1{X_i}$ is inconsistent and eliminated.
\end{itemize}
Hence, for any deterministic algorithm, there exists a realizable sequence such that the algorithm will make at least $k-1$ mistakes.

    Now consider $d$ independent copies of the star graphs and singletons. More specifically, consider $d(k+1)$ nodes $\{x_{i,j}|i = 1,\ldots,d, j=0,\ldots, k\}$, where $x_{i,0}$ is connected to $\{x_{i,j}|j=1,\ldots, k\}$ for all $i\in [\log k]$. So $\{x_{i,j}|j=0,\ldots, k\}$ compose a star graph.
    For each hypothesis in the hypothesis class $\cH$, it labels exactly one node in $\{x_{i,j}|j=1,\ldots, k\}$ as positive for all $i\in [d]$.
    Hence, $\cH$ has Littlestone dimension $\ld(\cH) = d$.
    Since every copy of the star graph-singletons is independent, we can show that any deterministic algorithm will make at least $d(k-1)$ mistakes.
\end{proof}

\section{Unknown Manipulation Graph Setting}\label{app:unknownGraph}
In this setting, the underlying graph $G^\star$ is unknown. Instead, the learner is given the prior knowledge of a finite graph class $\cG$. When $G^\star$ is undisclosed, we cannot compute $v_t$ given $x_t$. Consequently, the scenarios involving the observation timing of the features encompass the following: observing $x_t$ beforehand followed by $v_t$ afterward, observing $(x_t,v_t)$ afterward, and observing either $x_t$ or $v_t$ afterward, arranged in order of increasing difficulty. In this section, we provide results for the easiest case of observing $x_t$ beforehand followed by $v_t$ afterward. At the end, we will provide a negative result in the second easiest setting of observing $(x_t,v_t)$ afterward.

Since we not only have a hypothesis class $\cH$ but also a graph class $\cG$, we formally define realizability based on both the hypothesis class and the graph class as follows.

\begin{definition}[$(\cG,\cH)$-Realizability]\label{def:rel-graph}
A sequence of agents $(x_1,y_1),\ldots,(x_T,y_T)$ is $(\cG,\cH)$-realizable if there exists a graph $G\in \cG$ such that the neighborhood of $x_t$ in $G$ is identical to that in $G^\star$, i.e., $N_{G}(x_t) = N_{G^\star}(x_t)$ and there exists a perfect hypothesis $h^\star\in \cH$ satisfying $\lstr_{G^\star}(h,(x_t,y_t)) =0$ for all $t=1,\ldots, T$.

For any data distribution $\cD$, we say that $\cD$ is $(\cG,\cH)$-realizable if there exists a graph $G\in \cG$ such that $\PPs{(x,y)\sim \cD}{N_{G}(x)\neq N_{G^\star}(x)} =0$ and there exists a perfect hypothesis  $h^\star\in \cH$ s.t. $\Lstr_{G^\star,\cD}(h^\star)= 0$.
\end{definition}
    
\subsection{PAC Learning}
\paragraph{Realizable PAC Learning} 
We have described our algorithm in Section~\ref{sec:highlight-ug-pac}. Here, we restate the algorithm and the theorems with the proofs.

\underline{Prediction:} At each round $t$, after observing $x_t$, we implement the hypothesis $h(x)=\1{x\neq x_t}$, which labels all every point as positive except $x_t$. Then we can obtain a manipulated feature vector $v_t\sim \Unif(N_{G^\star}(x))$. 

\underline{Output:} 
Let $W$ denote the all graph-hypothesis pairs of $(G,h)$ satisfying that
\begin{align}
    \sum_{t=1}^T \1{v_t\notin N_G(x_t)} = 0 \,,\label{eq:lossG-app} \\
    \sum_{t=1}^T \1{\bar{h}_G(x_t)\neq y_t} = 0\,,\label{eq:lossGH-app}
\end{align}
where Eq~\eqref{eq:lossG-app} guarantees that every observed feature $v_t$ is a neighbor of $x_t$ in $G$ and Eq~\eqref{eq:lossGH-app} guarantees that $h$ has zero empirical strategic loss when the graph is $G$.

Let 
$$(\hat G, \hat h) = \argmin_{(G,h)\in W} \sum_{t=1}^T \abs{N_G(x_t)}$$
be the graph-hypothesis pair such that the graph has a \textbf{minimal empirical degree}. Finally, we output $\hat h$.

\ugPacRel*
\begin{proof}
    Since the manipulation is local (i.e., agents can only manipulate to their neighbors), if the neighborhood of $x$ is the same in two different graphs $G_1, G_2$, i.e., $N_{G_1}(x) = N_{G_2}(x)$, then for any implementation $h$, the induced labeling of $x$ is the same $\bar{h}_{G_1}(x) = \bar{h}_{G_2}(x)$.
Therefore, the strategic loss of $\hat h$ can be written as
\begin{align}
\Lstr_{G^\star,\cD}(\hat h)= &
    \PPs{(x,y)\sim \cD}{\bar{\hat h}_{G^\star}(x)\neq y}\nonumber\\
    \leq& \PPs{(x,y)\sim \cD}{N_{G^\star}(x)\neq N_{\hat G}(x)} +\PPs{(x,y)\sim \cD}{\bar{\hat h}_{\hat G}(x)\neq y \wedge N_{G^\star}(x)= N_{\hat G}(x) }\nonumber\\
    \leq& \PPs{(x,y)\sim \cD}{N_{G^\star}(x)\neq N_{\hat G}(x)} +\PPs{(x,y)\sim \cD}{\bar{\hat h}_{\hat G}(x)\neq y}\,.\label{eq:loss}
\end{align}
We bound the second term by uniform convergence since its empirical estimate is zero (see Eq~\eqref{eq:lossGH-app}) in Lemma~\ref{lmm:pair-term} and the first term in Lemma~\ref{lmm:neighbor-term}.
\end{proof}

\begin{lemma}\label{lmm:pair-term}
    With probability at least $1-\delta$ over $(x_1,y_1),\ldots,(x_T,y_T)\sim \cD^T$, for all $(h,G)$ satisfying $\frac{1}{T}\sum_{t=1}^T \1{\bar{h}_{G}(x_t)\neq y_t}=0$, we have 
    $$\Pr_{(x,y)\sim \cD}(\bar{h}_{G}(x)\neq y) \leq \epsilon\,,$$
    when $T= O(\frac{(d \log(kd) + \log|\cG|)\log(1/\epsilon) + \log(1/\delta)}{\epsilon})$.
    
    Hence, we have
    $$\Pr_{(x,y)\sim \cD}(\bar{\hat h}_{\hat G}(x)\neq y) \leq \epsilon\,.$$
\end{lemma}
\begin{proof}
Let $\ell_{\text{pair}}(h,G) := \Pr_{(x,y)\sim \cD}(\bar{h}_{G}(x)\neq y)$ denote the loss of the hypothesis-graph pair and $\hat \ell_{\text{pair}}(h,G) = \frac{1}{T}\sum_{t=1}^T \1{\bar{h}_{G}(x_t)\neq y_t}$ denote the corresponding empirical loss.
To prove the lemma, it suffices to bound the VC dimension of the composite function class $\cH\circ \cG := \{\bar{h}_{G}|h\in \cH, G\in \cG\}$ and then apply the tool of VC theory.

Since for any $n$ points, the number of labeling ways by $\cH\circ \cG$ is upper bounded by the sum of the number of labeling ways by $\bar \cH_G$ over all graphs $G\in \cG$. As discussed in the proof of Theorem~\ref{thm:vcd-induced}, for any fixed $G$ with maximum out-degree $k$, the number of labeling ways by $\bar \cH_G$ is at most $O((nk)^d)$. Therefore, the total number of labeling ways over $n$ points by $\cH\circ \cG$ is at most $O(|\cG|(nk)^d)$. Hence, we have $\vcd(\cH\circ \cG) = O(d \log(kd) + \log(|\cG|))$.
Therefore, by VC theory, we have
\begin{align*}
    \Pr(\ell_{\text{pair}}(\hat h,\hat G) >\epsilon)\leq \Pr(\exists (h,G)\in \cH\times \cG \text{ s.t. }\ell_{\text{pair}}(h, G) >\epsilon, \hat \ell_{\text{pair}}(h, G)=0)\leq \delta
\end{align*}
when $T= O(\frac{(d \log(kd) + \log|\cG|)\log(1/\epsilon) + \log(1/\delta)}{\epsilon})$.

\end{proof}
\begin{lemma}\label{lmm:neighbor-term}
    With probability at least $1-\delta$ over $x_{1:T},v_{1:T}$, for all $G$ satisfying $\frac{1}{T}\sum_{t=1}^T\1{N_{G^\star}(x_t)\neq N_{G}(x_t)} = 0$, we have 
    $$\Pr_{(x,y)\sim \cD}(N_{G^\star}(x)\neq N_{G}(x)) \leq \epsilon\,,$$
when $T\geq \frac{8k\log(|\cG|/\delta)}{\epsilon}$. 

Hence, we have    
    $$\Pr_{(x,y)\sim \cD}(N_{G^\star}(x)\neq N_{\hat G}(x)) \leq \epsilon\,.$$
\end{lemma}

\begin{proof}[of Lemma~\ref{lmm:neighbor-term}]
Let $\lgraph (G) = \Pr_{(x,y)\sim \cD}(N_{G^\star}(x)\neq N_{G}(x))$ denote the loss of a graph $G$, which is the 0/1 loss of neighborhood prediction.
Let $\hatlgraph (G) = \frac{1}{T}\sum_{t=1}^T\1{N_{G^\star}(x_t)\neq N_{G}(x_t)}$ denote the corresponding empirical loss.
It is hard to get an unbiased estimate of the loss  $\lgraph (G)$ since we cannot observe the neighborhood of any sampled $x$. 
However, it is easy for us to observe a $v \in N_{G^\star}(x)$ and remove all inconsistent $G$. Then the challenge is: how can we figure out the case of a strictly larger graph? This corresponds to the case that $N_{G^\star}(x)$ is a strict subset of $N_{G}(x)$. We deal with this case by letting $\hat G$ be the ``smallest'' consistent graph.
\begin{claim}\label{cla:neighborhood-loss}
Suppose that $G^\star$ and all graphs $G\in \cG$ have maximum out-degree at most $k$. For any $G\in \cG$ with the minimal empirical degrees, i.e., $\sum_{t=1}^T\abs{N_{G}(x_t)} - \abs{N_{G^\star}(x_t)} \leq 0$, we have 
    $$\hatlgraph (G)\leq \frac{2k}{T}\sum_{t=1}^T \Pr_{v\sim \Unif(N_{G^\star}(x_t)) }( v\notin N_{G}(x_t))\,.$$
\end{claim}
\begin{proof}
We decompose
\begin{align}
    \1{N_{G^\star}(x)\neq N_{G}(x)}
    \leq \abs{N_{G}(x)\setminus N_{G^\star}(x)} + \abs{N_{G^\star}(x)\setminus N_{G}(x)}\,.\label{eq:decompose}
\end{align}

    Since $\sum_{t=1}^T\abs{N_{G}(x_t)} \leq \sum_{t=1}^T \abs{N_{G^\star}(x_t)}$, we have 
\begin{align}
    \sum_{t=1}^T\abs{N_{G}(x_t)\setminus N_{G^\star}(x_t)} \leq \sum_{t=1}^T\abs{N_{G^\star}(x_t)\setminus N_{G}(x_t)}\,,\label{eq:mindegree}
\end{align}
by deducting $\sum_{t=1}^T |N_{G}(x_t)\cap N_{G^\star}(x_t)|$ on both sides.

By combining Eqs~\eqref{eq:decompose} and \eqref{eq:mindegree}, we have
\begin{align*}
    \hatlgraph (G) =&\frac{1}{T}\sum_{t=1}^T \1{N_{G^\star}(x_t)\neq N_{G}(x_t)}\\
    \leq &  \frac{2}{T}\sum_{t=1}^T\abs{N_{G^\star}(x_t)\setminus N_{G}(x_t)}\\
    = & \frac{2}{T}\sum_{t=1}^T\sum_{v\in N_{G^\star}(x_t)}\1{v\notin N_{G}(x_t)}\\
    = & \frac{2}{T}\sum_{t=1}^T \abs{N_{G^\star}(x_t)}\Pr_{v\sim \Unif(N_{G^\star}(x_t)) }( v\notin N_{G}(x_t))\\
    \leq & \frac{2}{T}\sum_{t=1}^T k\Pr_{v\sim \Unif(N_{G^\star}(x_t)) }( v\notin N_{G}(x_t))
    \,.
\end{align*}
\end{proof}
Now we have connected $\hatlgraph (G)$ with $\frac{1}{T}\sum_{t=1}^T \Pr_{v\sim \Unif(N_{G^\star}(x_t)) }( v\notin N_{G}(x_t))$, where the latter one is estimable. 
Since $\hat G$ is a consistent graph, we have
$$\frac{1}{T}\sum_{t=1}^T \1{ v_t\notin N_{\hat G}(x_t))} =0\,,$$
which is an empirical estimate of $\frac{1}{T}\sum_{t=1}^T \Pr_{v\sim \Unif(N_{G^\star}(x_t)) }( v\notin N_{\hat G}(x_t))$.
Then by showing that this loss is small, we can show that  $\hatlgraph (\hat G)$ is small.
\begin{claim}\label{cla:empirical-neighbor}
    Suppose that $G^\star$ and all graphs $G\in \cG$ have maximum out-degree at most $k$ and $\cD$ is $(\cG,\cH)$-realizable. For any fixed sampled sequence of $x_{1:T}$, with probability at least $1-\delta$ over $v_{1:T}$ (where $v_t$ is sampled from $\Unif(N_{G^\star}(x_t))$), we have 
    $$\hatlgraph (\hat G)\leq\epsilon\,,$$
    when $T\geq \frac{14k\log(|\cG|/\delta)}{3\epsilon}$.
\end{claim}
\begin{proof}
According to Claim~\ref{cla:neighborhood-loss}, we only need to upper bound $\frac{1}{T}\sum_{t=1}^T \Pr_{v\sim \Unif(N_{G^\star}(x_t)) }( v\notin N_{\hat G}(x_t))$ by $\frac{\epsilon}{2k}$.
For any graph $G$, by empirical Bernstein bounds~(Theorem 11 of \cite{maurer2009empirical}), with probability at least $1-\delta$ over $v_{1:T}$ we have
\begin{align*}
    \frac{1}{T}\sum_{t=1}^T \Pr_{v\sim \Unif(N_{G^\star}(x_t)) }( v\notin N_{G}(x_t))\leq \frac{1}{T}\sum_{t=1}^T \1{v_t\notin N_{G}(x_t)} + \sqrt{\frac{2V_{G,T}(v_{1:T})\log(2/\delta)}{T}} +\frac{7\log(1/\delta)}{3T}\,,
\end{align*}
where $V_{G,T}(v_{1:T}) := \frac{1}{T(T-1)} \sum_{t,\tau=1}^T \frac{(\1{v_t\notin N_{G}(x_t)}-\1{v_\tau\notin N_{G}(x_\tau)} )^2}{2}$ is the sample variance.
Since $\hat G$ is consistent, we have $\frac{1}{T}\sum_{t=1}^T \1{v_t\notin N_{\hat G}(x_t)} =0$ and $V_{\hat G,T}(v_{1:T}) =0$.
Then by union bound over all $G\in \cG$, we have that with probability at least $1-\delta$,
\begin{align*}
    \frac{1}{T}\sum_{t=1}^T \Pr_{v\sim \Unif(N_{G^\star}(x_t)) }( v\notin N_{\hat G}(x_t))\leq \frac{7\log(|\cG|/\delta)}{3T}\,.
\end{align*}
\end{proof}

Now we have that for any fixed sampled sequence of $x_{1:T}$, w.p. at least $1-\delta$ over $v_{1:T}$,  $\hatlgraph (\hat G)\leq\epsilon$. We will apply concentration inequality over $x_{1:T}$ to show that $ \lgraph (\hat G)$ is small and then finish the proof of Lemma~\ref{lmm:neighbor-term}.
\begin{align*}
    &\Pr_{x_{1:T},v_{1:T}}(\lgraph (\hat G) >2\epsilon) \\
    \leq & \Pr_{x_{1:T},v_{1:T}}(\lgraph (\hat G) >2\epsilon, \hatlgraph (\hat G)\leq\epsilon) + \Pr_{x_{1:T},v_{1:T}}(\hatlgraph (\hat G)>\epsilon) \\
    \leq& \Pr_{x_{1:T}}(\exists G\in \cG, \lgraph (G) >2\epsilon,\hatlgraph ( G)\leq\epsilon) + \delta \tag{Claim~\ref{cla:empirical-neighbor}}\\
    \leq & 2\delta\,,
\end{align*}
where the last inquality holds by Chernoff bounds when $T\geq \frac{8}{\epsilon} (\log(|\cG|/\delta))$. Hence, combined with Claim~\ref{cla:empirical-neighbor}, we need $T \geq O(\frac{k\log(|\cG|/\delta)}{\epsilon})$ overall. 

\end{proof}

\thmUgPacRelLb*
\begin{proof}
    Consider the input space of one node $o$ and $n(n+1)$ nodes $\{x_{ij}|i=0,\ldots,n, j=1,\ldots,n\}$, which are partitioned into $n$ subsets $X_0 = \{x_{01}, x_{02},...x_{0n}\}, X_1 = \{x_{11}, x_{12},...x_{1n}\},..., X_n$. The hypothesis will label one set in $\{X_1,\ldots,X_n\}$ by positive and the hypothesis class is $H = \{\1{X_i}|i\in [n]\}$. The target function is $\1{X_{i^*}}$. This class is analogous to singletons if we view each group as a composite node. However, since the degree of the manipulation graph is limited to $1$, we split one node into $n$ copies.
    The marginal data distribution put probability mass $1-2\epsilon$ on the irrelevant node $o$ and the remaining $2\epsilon$ uniformly over $X_0$.

    Let the graph class $\cG$ be the set of all graphs which connect $x_{0i}$ to at most one node in $\{ x_{1i},\ldots,x_{ni}\}$ for all $i$. So the cardinality of the graph class $\cG$ is $|\cG| = (n+1)^n$.
    
    When the sample size $T$ is smaller than $\frac{n}{8\epsilon}$, we can sample at most $n/2$ examples from $X_0$ with constant probability by Chernoff bounds. 
    Then there are at least $n/2$ examples in $X_0$ that have not been sampled.

    W.l.o.g., let $x_{01},\ldots,x_{0\frac{n}{2}}$ denote the sampled examples. The graph $G$ does not put any edge on these sampled examples. So all these examples are labeled as negative no matter what $i^*$ is. 
    Then looking at the output $\hat h$ at any unseen example $x_{0j}$ in $X_0$.
    \begin{itemize}
    \item If $\hat h$ predicts all points in $\{x_{0j}, x_{1j},\ldots,x_{nj}\}$ as $0$, we add an edge between $x_{0j}$ and $x_{i^*j}$ to $G$. So $\hat h$ misclassfy $x_{0j}$ under manipulation graph $G$.
    \item If $\hat h$ predicts $x_{0j}$ as positive, then we do not add any edge on $x_{0t}$ in $G$.  So $\hat h$ will classify $x_{0j}$ as $1$ but the target hypothesis will label $x_{0j}$ as $0$.
    \item If $\hat h$ predicts any node $x_{ij}$ with $i\neq i^*\in [n]$ as positive, we add an edge between $x_{0j}$ and $x_{ij}$. So $\hat h$ will classify $x_{0j}$ as $1$ but the target hypothesis will label $x_{0j}$ as $0$.
    \item If $\hat h$ predicts exactly $x_{i^*j}$ in $\{x_{0j}, x_{1j},\ldots,x_{nj}\}$ as positive. Then since we can arbitrarily pick $i^*$ at the beginning, there must exist at least one $i^*$ such that at most this case will not happen.
\end{itemize}
Therefore, $\hat h$ misclassfy every unseen point in $X_0$ under graph $G$ and $\Lstr_{G,\cD}(\hat h)\geq \epsilon$.
\end{proof}

\paragraph{Agnostic PAC Learning} 
Now, we explore the agnostic setting where there may be no perfect graph in $\cG$, no perfect hypothesis $h^\star\in \cH$, or possibly neither. In the following, we first define the loss of the optimal loss of $\cH$ and the optimal loss of $\cG$ and aim to find a predictor with a comparable loss.

\begin{definition}[Optimal loss of $\cH$]
Let $\Delta_\cH$ denote the optimal strategic loss achievable by hypotheses in $\cH$. That is to say,  $$\Delta_\cH := \inf_{h\in \cH} \Lstr_{G^\star,\cD}(h)\,.$$
\end{definition}

\begin{definition}[Optimal loss of $\cG$]\label{def:opt-loss-g}
Let $\Delta_\cG$ denote the graph loss (0/1 loss of neighborhood) of the optimal graph $G^\dagger\in \cG$. 
That is to say,  $$\Delta_\cG := \min_{G^\dagger\in \cG} \lgraph (G)= \min_{G^\dagger\in \cG} \Pr_{(x,y)\sim \cD}(N_{G^\star}(x)\neq N_{G}(x))\,.$$
\end{definition}


We start introducing our algorithm by providing the following lemma, which states that if we are given an approximately good graph, then by minimizing the strategic loss under this approximate graph, we can find an approximately good hypothesis.
\begin{lemma} \label{lmm:given-app-graph}
    Given a graph $G$ with the loss of neighborhood being $\lgraph (G)=\alpha$, for any $h$ being $\epsilon$-approximate optimal under manipulation graph $G$ is the true graph, i.e., $h$ satisfies
    $\Lstr_{G,\cD}(h)\leq \inf_{h^\star\in \cH}\Lstr_{G,\cD}(h^\star) +\epsilon$,
    it must satisfy
    $$\Lstr_{G^\star,\cD}(h) \leq 2\alpha + \Delta_\cH +\epsilon\,.$$
\end{lemma}
\begin{proof}
By definition, we have
    \begin{align*}
        \Lstr_{G^\star,\cD}(h)  =&\PPs{(x,y)\sim \cD}{\bar {h}_{G^\star}(x)\neq y} \\
        = &\PPs{(x,y)\sim \cD}{\bar {h}_{G}(x)\neq y \wedge N_{G^\star}(x) = N_{G}(x)} + \PPs{(x,y)\sim \cD}{\bar {h}_{G^\star}(x)\neq y \wedge N_{G^\star}(x) \neq N_{G}(x)}\\
        \leq &\PPs{(x,y)\sim \cD}{\bar {h}_{G}(x)\neq y} + \PPs{(x,y)\sim \cD}{N_{G^\star}(x) \neq N_{G}(x)}\\
        \leq &\PPs{(x,y)\sim \cD}{\bar {h^\star}_{G}(x)\neq y} +\epsilon + \alpha\tag{$\epsilon$-approximate optimality of $h$}\\
        = &\PPs{(x,y)\sim \cD}{\bar {h^\star}_{G^\star}(x)\neq y \wedge N_{G^\star}(x) = N_{G}(x)} + \PPs{(x,y)\sim \cD}{N_{G^\star}(x) \neq N_{G}(x)} +\epsilon+ \alpha\\
        \leq &\PPs{(x,y)\sim \cD}{\bar {h^\star}_{G^\star}(x)\neq y} +\epsilon+ 2\alpha =\Delta_\cH +\epsilon+ 2\alpha\,.
    \end{align*}
\end{proof}

Then the remaining question is: \textbf{How to find a good approximate graph?} As discussed in the previous section, the graph loss $\lgraph(\cdot)$ is not estimable. Instead, we construct a proxy loss that is not only close to the graph loss but also estimable. 
We consider the following alternative loss function as a proxy:
$$\lproxy(G) = 2\mathbb{E}_{x}[P_{v\sim N_{G^\star}(x)}(v\notin N_G(x))] + \frac{1}{k}\mathbb{E}[|N_G(x)|] - \frac{1}{k}\mathbb{E}[|N_{G^\star}(x)|]\,,$$
where $k$ is the degree of graph $G^\star$. Note that the first two terms are estimable and the third term is a constant. Hence, we can minimize this proxy loss.

\lmmproxyLoss*
\begin{proof}
Let $d_G = \mathbb{E}[|N_G(x)|] - \mathbb{E}[|N_{G^\star}(x)|]$ denote the difference between the expected degree of $G$ and that of $G^\star$.
We have 
\begin{equation}
    d_G = \mathbb{E}_{(x,y)\sim D}[|N_G(x)|] - \mathbb{E}_{(x,y)\sim D}[|N_{G^\star}(x)|] = \mathbb{E}[|N_G(x)\setminus N_{G^\star}(x)|] - \mathbb{E}[|N_{G^\star}(x)\setminus N_{G}(x)|]\,.\label{eq:dg}
\end{equation}
 

Then, we have 
\begin{align*}
    \lproxy(G) =& 2\mathbb{E}_{x}[\frac{|N_{G^\star}(x)\setminus N_{G}(x)|}{|N_{G^\star}(x)|}] +\frac{1}{k} d_G \\
    \geq& 2\cdot \frac{1}{k}\mathbb{E}_{x}[|N_{G^\star}(x)\setminus N_{G}(x)|] +\frac{1}{k} d_G \\
    =& \frac{1}{k} (\mathbb{E}[|N_G(x)\setminus N_{G^\star}(x)| + |N_{G^\star}(x)\setminus N_{G}(x)|])\tag{Applying Eq~\eqref{eq:dg}}\\
    \geq & \frac{1}{k} \lgraph (G)\,.
\end{align*}

On the other hand, we have
\begin{align*}
    \lproxy(G) =& 2\mathbb{E}_{x}[\frac{|N_{G^\star}(x)\setminus N_{G}(x)|}{|N_{G^\star}(x)|}] +\frac{1}{k} d_G \\ 
    \leq & 2\EEs{x}{\1{N_{G}(x) \neq N_{G^\star}(x)}} +\EEs{x}{\1{N_{G}(x) \neq N_{G^\star}(x)}}\\
    =& 3 \lgraph(G)\,.
\end{align*}
\end{proof}

Given a sequence of $S = ((x_1,v_1),\ldots,(x_T,v_T))$, we define the empirical proxy loss over the sequence $S$ as 
$$\hatlproxy (G, S) = \frac{2}{T}\sum_{t=1}^T\1{v_t \notin N_G(x_t)} + \frac{1}{kT}\sum_{t=1}^T |N_G(x_t)|- \frac{1}{k}\EEs{(x,y)\sim D}{|N_{G^\star}(x)|}$$
where the first two terms are the empirical estimates of the first two terms of $\lproxy(G)$ and the last term is not dependent on $G$, which is a constant.

Similar to the realizable setting, by implementing the hypothesis $h_t=\1{x\neq x_t}$, which labels every point as positive except $x_t$, we observe the manipulated feature vector $v_t\sim \Unif(N_{G^\star}(x))$. 

Given two samples, $S_1 = ((x_1,v_1,y_1),\ldots, (x_{T_1},v_{T_1},y_{T_1}))$ and $S_2=((x_1',v_1',y_1'),\ldots, (x_{T_2}',v_{T_2}',y_{T_2}'))$, we use $S_1$ to learn an approximate good graph and $S_2$ to learn the hypothesis.
Let $\hat G$ be the graph minimizing the empirical proxy loss, i.e., $$\hat G = \argmin_{G\in \cG} \hatlproxy(G, S_1)= \argmin_{G\in \cG} \frac{2}{T}\sum_{t=1}^{T_1}\1{v_t \notin N_G(x_t)} + \frac{1}{kT}\sum_{t=1}^{T_1} |N_G(x_t)|\,.$$
Then let $\hat h$ be the ERM implementation assuming that $\hat G$ is the true graph, i.e.,
$$\hat h =\argmin_{h\in \cH} \frac{1}{T_2}\sum_{t=1}^{T_2} \1{\bar h_{\hat G} (x_t') \neq y_t'}\,.$$
Next, we bound the strategic loss of $\hat h$ using $\epsilon,\Delta_\cG$ and $\Delta_\cH$.
\begin{theorem}\label{thm:ug-pac-agn}
For any hypothesis class $\cH$ with $\vcd(\cH)=d$, the underlying true graph $G^\star$ with maximum out-degree at most $k$, finite graph class $\cG$ in which all graphs have maximum out-degree at most $k$, any data distribution, and any $\epsilon,\delta\in (0,1)$, with probability at least $1-\delta$ over $S_1\sim \cD^{T_1}, S_2\sim \cD^{T_2}$ and $v_{1:T_1}$ where $T_1=O( \frac{k^2\log(|\cG|/\delta)}{\epsilon^2})$ and $T_2= O( \frac{d\log(kd) + \log(1/\delta)}{\epsilon^2})$
, the output $\hat h$ satisfies
\begin{equation*}
        \Lstr_{G^\star,\cD}(\hat h)\leq 6k\Delta_\cG  +\Delta_\cH + \epsilon\,.
\end{equation*}
\end{theorem}
\begin{proof}
    We first prove that $\lgraph(\hat G) \leq 3k \Delta_\cG + 2\epsilon$.
    
    By Hoeffding bounds and union bound, with probability at least $1-\delta/2$ over $S_1$, for all $G\in \cG$,
    \begin{equation}
        |\hatlproxy(G, S_1) -\lproxy(G)| \leq \epsilon_1\,,
    \end{equation}
    when $T_1= O(\frac{\log(|\cG|/\delta)}{\epsilon_1^2})$.
    
    Hence, we have
    \begin{align*}
        \lgraph(\hat G)
        \leq & k \lproxy(\hat G)\tag{Applying Lemma~\ref{lmm:proxy-loss}}\\
        \leq & k\hatlproxy(\hat G, S_1) +k \epsilon_1\\
        \leq & k\hatlproxy(G^\dagger, S_1) + k \epsilon_1\\
        \leq & k\lproxy(G^\dagger) +2k \epsilon_1\\
        \leq & 3k \lgraph(G^\dagger) + 2k \epsilon_1\tag{Applying Lemma~\ref{lmm:proxy-loss}}\\
        = & 3k \Delta_\cG + 2k \epsilon_1\,.\tag{Def~\ref{def:opt-loss-g}}
    \end{align*}

    Then by VC theory and uniform convergence, with probability at least $1-\delta/2$ over $S_2$, for all $h\in \cH$, 
    \begin{equation*}
        |\hatlstr_{\hat G}(h) - \Lstr_{\hat G,\cD}(h)| \leq \epsilon_2\,,
    \end{equation*}
    when $T_2=O( \frac{d\log(kd) + \log(1/\delta)}{\epsilon_2^2})$.
    
    Therefore $\hat h$ is an approximately good implementation if $\hat G$ is the true graph, i.e.,
    \begin{equation*}
        \Lstr_{\hat G,\cD}(\hat h)\leq \min_{h\in \cH}\Lstr_{\hat G,\cD}(h) + 2\epsilon_2\,.
    \end{equation*}
    Then by applying Lemma~\ref{lmm:given-app-graph}, we get
    \begin{equation*}
        \Lstr_{G^\star,\cD}(\hat h)\leq 6k\Delta_\cG  +\Delta_\cH +4k\epsilon_1+ 2\epsilon_2\,. 
    \end{equation*}  
    We complete the proof by plugging in $\epsilon_1 = \frac{\epsilon}{8k}$ and $\epsilon_2 = \frac{\epsilon}{4}$.
\end{proof}


\subsection{Application of the Graph Learning Algorithms in Multi-Label Learning}
Our graph learning algorithms have the potential to be of independent interest in various other learning problems, including multi-label learning. To illustrate, let us consider scenarios like the recommender system, where we aim to recommend movies to users. In such cases, each user typically has multiple favorite movies, and our objective is to learn this set of favorite movies.
Similarly, in the context of object detection in computer vision, every image often contains multiple objects. Here, our learning goal is to accurately identify and output the set of all objects present in a given image.

This multi-label learning problem can be effectively modeled as a bipartite graph denoted as $G^\star = (\cX, \cY, \cE^\star)$, where $\cX$ represents the input space (in the context of the recommender system, it represents users), $\cY$ is the label space (in this case, it signifies movies), and $\cE^\star$ is a set of edges. In this graph, the presence of an edge $(x, y) \in \cE^\star$ implies that label $y$ is associated with input $x$ (e.g., user $x$ liking movie $y$). Our primary goal here is to learn this graph, i.e., the edges $\cE^\star$. More formally, given a marginal data distribution $\cD_\cX$, our goal is to find a graph $\hat G$ with minimum neighborhood prediction loss $\lgraph (G)= \Pr_{x\sim \cD_\cX}(N_{G}(x)\neq N_{G^\star}(x))$. Suppose we are given a graph class $\cG$, then our goal is to find a graph $G^\dagger$ such that 
\[G^\dagger = \argmin_{G\in \cG}\lgraph (G).\]

However, in real-world scenarios, the recommender system cannot sample a user along with her favorite movie set.
Instead, at each time $t$, the recommender recommends a set of movies $h_t$ to the user, and the user randomly clicks on one of the movies in $h_t$ that she likes (i.e., $v_t\in N_{G^\star}(x_t)\cap h_t$). Here we abuse the notation a bit by letting $h_t$ represent the set of positive points labeled by this hypothesis.
This setting perfectly fits into our unknown graph setting.

Therefore, in the realizable setting where $\lgraph (G^\dagger) =0$, we can find $\hat G$ satisfying $\lgraph (\hat G) \leq \epsilon$ given 
$O(\frac{8k\log(|\cG|/\delta)}{\epsilon})$ examples according to Lemma~\ref{lmm:neighbor-term} .
In the agnostic setting, we can find $\hat G$ satisfying $\lgraph (\hat G) \leq 3k\lgraph (G^\dagger) + \epsilon$ given $O(\frac{k^2 \log(|\cG|/\delta)}{\epsilon^2})$ examples according to Theorem~\ref{thm:ug-pac-agn}.

\subsection{Online Learning}
In the unknown graph setting, we do not observe $N_{G^\star}(x_t)$ or $N_{G^\star}(v_t)$ anymore.
We then design an algorithm by running an instance of Algorithm~\ref{alg:reduction2online-pmf} over the majority vote of the graphs.
\begin{algorithm}[H]\caption{Online Algorithm in the Unknown Graph Setting}\label{alg:ug-online}
    \begin{algorithmic}[1]
        \STATE initialize an instance $\cA=$ Red2Online-PMF (SOA, $2k$) 
        \STATE let $\cG_1 = \cG$
        \FOR{$t=1,2,\ldots$}
        \STATE \underline{Prediction}: after observing $x_t$, for every node $x$
        \IF{$(x_t,x)$ are connected in at most half of the graphs in $\cG_t$}
        \STATE $h_t(x) = 1$\label{alg-line:type1}
        \ELSE
        \STATE let the prediction at $x$ follow $\cA$, i.e., $h_t(x) = \cA(x)$ \label{alg-line:type2}
        \ENDIF
        \STATE \underline{Update}:  when we make a mistake at the observed node $v_t$:
        \IF{$(x_t,v_t)$ are connected in at most half of the graphs in $\cG_t$}
        \STATE update $\cG_{t+1} = \{G\in \cG_t| (x_t,v_t) \text{ are connected in }G\}$ to be the set of consistent graphs
        \ELSE
        \STATE feed $\cA$ with $(v_t, \tilde N(v_t), \hat y_t, y_t)$ with
       $\tilde N(v_t) = \{x| \abs{\{G\in \cG_t|(v_t,x) \text{ are connected in } G\}} \geq |\cG_t|/2\}$, which is the set of vertices which are an out-neighbor of $v_t$ in more than half of the graphs in $\cG_t$
        \STATE $\cG_{t+1} = \cG_t$
        \ENDIF
        \ENDFOR
    \end{algorithmic}
\end{algorithm}

\ugOnline*
\begin{proof}
Let $W_t$ denote the total weight of experts at the beginning of round $t$ in the algorithm instance $\cA$. 
When we make a mistake at round $t$, there are two cases.
\begin{itemize}
    \item Type 1: $(x_t,v_t)$ are connected in at most half of the graphs in $\cG_t$. In this case, we can remove half of the graphs in $\cG_t$, i.e., $\abs{\cG_{t+1}}\leq \frac{\abs{\cG_t}}{2}$. The total number of mistakes in this case will be $N_1\leq \log(\abs{\cG})$.

    \item Type 2: $(x_t,v_t)$ are connected in more than half of the graphs in $\cG_t$. In this case, our prediction rule applies $\cA$ to predict at $v_t$. Then there are two cases (as we did in the proof of Theorem~\ref{thm:fu-online})
    \begin{itemize}
        \item False positive: meaning that we predicted the neighbor $v_t$ of $x_t$ by $1$ but the correct prediction is $0$. This means that the neighbor $v_t$ we saw should also be labeled as $0$ by the target implementation $h^\star$. In this case, since Algorithm~\ref{alg:reduction2online-pmf} does not use the neighborhood to update, and it does not matter what $\tilde N_{G^\star}(v_t)$ is.
    \item False negative: meaning we predicted $x_t$ with $0$ but the correct prediction is $1$. In this case, $v_t=x_t$, the entire neighborhood $N_{G^\star}(x_t)$ is labeled as $0$ by $h_t$, and $h^\star$ labeled some point in $N_{G^\star}(x_t)$ by $1$.
    Since only nodes in $\tilde N(v_t)$ are labeled as $0$ by our algorithm (as stated in line~\ref{alg-line:type1}, all nodes not in $\tilde N(v_t)$ are labeled as $1$ by $h_t$),  the true neighborhood $N_{G^\star}(v_t)$ must be a subset of $\tilde N(v_t)$. Since all graphs in $\cG$ has maximum out-degree at most $k$, we have $|\tilde N(v_t)|\leq \frac{k\cdot |\cG_t|}{|\cG_t|/2} = 2k$.  
    \end{itemize}
    Then, by repeating the same analysis of Theorem~\ref{thm:fu-online} (since the analysis only relies on the fact that the observed $N_{G^\star}(v_t)$ satisfy $|N_{G^\star}(v_t)|\leq k$), we will make at most $N_2\leq O(k\log k \cdot \ld(\cH))$ type 2 mistakes.

\end{itemize}
Therefore, we will make at most $N_1+N_2 \leq  O(\log\abs{\cG}+k\log k \cdot \ld(\cH))$ mistakes.
\end{proof}

\ugOnlineLb*
\begin{proof}
The example is very similar to the one in the proof of Theorem~\ref{thm:ug-pac-rel-lb}.
    Consider the input space of $n(n+1)$ nodes, which are partitioned into $n+1$ subsets $X_0 =\{x_{01}, x_{02},...x_{0n}\}, X_1 = \{x_{11}, x_{12},...x_{1n}\},..., X_n$. The hypothesis will label one set in $\{X_1,\ldots,X_n\}$ by positive and the hypothesis class is $H = \{\1{X_i}|i\in [n]\}$. The target function is $\1{X_{i^*}}$. This class is analogous to singletons if we view each group as a composite node. However, since the degree of the manipulation graph is limited to $1$, we split one node into $n$ copies.
    
    The agent will always pick an agent in $X_0$ and the true label is positive only when it is connected to $X_{i^*}$. To ensure that the degree of every node is at most $1$, only one node in each set will only be used in round $t$, i.e., $\{x_{0t}, x_{1t},\ldots,x_{nt}\}$.

At round $t$:
\begin{itemize}
    \item If the learner predicts all points in $\{x_{0t}, x_{1t},\ldots,x_{nt}\}$ as $0$, the adversary picks the agent $x_{0t}$ and adds an edge between $x_{0t}$ and $x_{i^*t}$. The agent does not move, the learner predicts $0$, and the true label is $1$ no matter what $i^*$ is. We do not learn anything about the target function.
    \item If the learner predicts $x_{0t}$ as positive, then the adversary picks the agent  $x_{0t}$  and does not add any edge on  $x_{0t}$. The agent does not move, the learner predicts $1$, and the true label is $0$. We learn nothing.
    \item If the learner predicts any node $x_{it}$ with $i\neq i^*\in [n]$ as positive, the adversary picks the agent  $x_{it}$  and adds no edge  on $x_{0t}$. The agent will stay at $x_{it}$, the learner predicts $1$, and the true label is $0$. But we can only eliminate one hypothesis $\1{X_i}$.
\end{itemize}
Hence there must exist an $i^*$ the algorithm will make at least $n-1$ mistakes.
Since all possible graphs will add at most one edge between $x_{0t}$ and $\{x_{1t},\ldots,x_{nt}\}$, the graph class $\cG$ has at most $(n+1)^n$ graphs.
\end{proof}
    
Next, we restate Proposition~\ref{prop:online-v-lb}. 
\uglessinfo*
\begin{proof}
    Consider $n+2$ nodes, $A,B$ and $C_1, C_2,\ldots, C_n$. 
    The graph class $\cG$ has $n$ graphs in the form of $A-B-C_i$ for $i\in [n]$. 
    The hypothesis class is singletons over $C_1,...,C_n$. 
    So we have $n$ graphs with degree $2$ and $n$ hypotheses. 
Suppose that the true graph is $A-B-C_{i^*}$ and the target function is $\1{C_{i^*}}$. Hence only the agents $B$ and $C_{i^*}$ are positive.

Then at each round:
\begin{itemize}
    \item If the learner is all-negative, the adversary picks agent $(B,1)$. The learner makes a mistake at 
    $x_t = v_t =B$, and learn nothing.
    \item If the learner implements $h_t$ satisfying $h_t(B)= 1$ (or $h_t(A)= 1$), then the adversary picks agent $(A,-1)$. The learner makes a mistake at $v_t = B$ (or $v_t = A$) and learns nothing.
    \item If the learner implements $h_t$ predicting any $C_i$ with $i \neq i^*$ by positive, then the adversary picks agent $(C_i, -1)$. The learner makes a mistake at $v_t = x_t = C_i$ and can only eliminate one hypothesis $\1{C_{i}}$ (and graph).
\end{itemize}
Hence, there must exist an $i^*$ such that the learner makes at least $n-1$ mistakes.
\end{proof}

\section{From Realizable Online Learning to Agnostic Online Learning}\label{app:online-agn}
Since we cannot achieve sublinear regret even in the standard online learning by deterministic algorithms, we have to consider randomized algorithms in agnostic online learning as well.
Following the model of ``randomized algorithms'' (instead of ``fractional classifiers'') by \cite{ahmadi2023fundamental},
for any randomized learner $p$ (a distribution over hypotheses), the strategic behavior depends on the realization of the learner. Then the strategic loss of $p$ is 
\begin{equation*}
    \lstr_G (p,(x,y)) = \EEs{h \sim p}{ \lstr_G(h,(x,y))}\,.
\end{equation*}

When we know $x_t$ and $N_{G^\star}(x_t)$, we actually have full information feedback (i.e., we can compute the prediction of every hypothesis ourselves). 
Hence, we can apply the realizable-to-agnostic technique by constructing a cover of $\cH$ with $T^M$ experts, where $M$ is the mistake bound in the realizable setting, and then running multiplicative weights over these experts.

For the construction of the experts, we apply the same method by~\cite{ben2009agnostic}. The only minor difference is that the learner in our setting needs to produce a hypothesis at each round to induce some labeling instead of deciding a labeling directly. Given the hypothesis $\tilde h_t$ generated by a realizable algorithm $\cA_{\text{REL}}$, to flip the prediction induced by $\tilde h_t$, we can change the entire neighborhood of $x_t$ to negative if $\tilde h_t$'s prediction is positive, and predict $x_t$ by positive if $\tilde h_t$'s prediction is negative. If we do not know $(x_t,N_{G^\star}(x_t))$, then we do not even know how to flip the prediction.

\begin{algorithm}[t]\caption{Construction of expert $\text{Exp}(i_1,\ldots,i_L)$ }\label{alg:expert-cover}
    \begin{algorithmic}[1]
    \STATE \textbf{Input: } an algorithm in the realizable setting $\cA_{\text{REL}}$ and indices $i_1,\ldots,i_L$
    \FOR{$t=1,\ldots, T$}{
    \STATE receive $x_t$
    \STATE let $\tilde h_t$ denote the hypothesis produced by $\cA_{\text{REL}}$, i.e., $\tilde h_t = \cA_{\text{REL}}((x_1,N_{G^\star}(x_1),\tilde y_1), \ldots,(x_{t-1},N_{G^\star}(x_{t-1}),\tilde y_{t-1}))$
    \IF{$t\in \{i_1,\ldots,i_L\}$}{
    \IF{$\tilde h_t$ labels the whole $N_{G^\star}(x_t)$ by negative}{
    \STATE define $h_t = \tilde h_t$ except $h_t(x_t) = 1$ and implement $h_t$
    \STATE let $\tilde y_t = 1$
    }
    \ELSE
    \STATE define $h_t$ s.t. $h_t$ labels the entire neighborhood $N_{G^\star}(x_t)$ by negative and implement $h_t$
    \STATE let $\tilde y_t = 0$
    \ENDIF
    }
    \ELSE   
    \STATE implement $\tilde h_t$ and let $\tilde y_t = \bar{\tilde h_t}_{G^\star}(x_t)$  
    \ENDIF
    }
    \ENDFOR
    \end{algorithmic}
\end{algorithm}
\begin{lemma}
    Let $\cA_{\text{REL}}$ be an algorithm that makes at most $M$ mistakes in the realizable setting. Let $x_1,\ldots,x_T$ be any sequence of instances. For any $h\in \cH$, there exists $L\leq M$ and indices $i_1,\ldots,i_L$, such that running $\text{Exp}(i_1,\ldots,i_L)$ (Algorithm~\ref{alg:expert-cover}) on the sequence $((x_1,N_{G^\star}(x_1)), \ldots, (x_T,N_{G^\star}(x_T)))$ generate the same prediction as implementing $h$ on each round, i.e., $\bar h_{tG^\star}(x_t) = \bar h_{G^\star}(x_t)$ for all $t$.
\end{lemma}
\begin{proof}
    Fix $h\in \cH$ and the sequence $((x_1,N_{G^\star}(x_1)), \ldots, (x_T,N_{G^\star}(x_T)))$. Then consider running $\cA_{\text{REL}}$ over $((x_1,N_{G^\star}(x_1), \bar h(x_1)), \ldots, (x_T,N_{G^\star}(x_T),\bar h(x_T)))$ and $\cA_{\text{REL}}$ makes $L\leq M$ mistakes. Define $\{i_1,\ldots,i_L\}$ to be the set of rounds $\cA_{\text{REL}}$ makes mistakes. Since the prediction of $\cA_{\text{REL}}$ is only different from $\bar h_{G^\star}$ at rounds in $\{i_1,\ldots,i_L\}$, the predictions when implementing $\text{Exp}(i_1,\ldots,i_L)$ are the same as implementing $h$.
\end{proof}
We finish with a result that holds directly by following~\cite{ben2009agnostic}.
\begin{proposition}
    \label{res:agn-online}
By running multiplicative weights over the experts, we can achieve regret
\begin{align*}
    \regret(T)\leq O(\sqrt{k\log k\cdot \ld(\cH) T\log(T)})\,.
\end{align*} 
\end{proposition}
It is unclear to us how to design algorithms in post-manipulation feedback and unknown graph settings.


\end{document}